\documentclass{article} 
\usepackage{nips14submit_e,times}
\usepackage{hyperref}
\usepackage{url}

\usepackage[numbers]{natbib}
\usepackage{enumerate}
\usepackage{enumitem}
\usepackage{bm}
\usepackage{amssymb}
\usepackage{amsthm}
\usepackage{amsmath}
\usepackage{graphicx}
\usepackage{mathtools}
\usepackage[stable]{footmisc}
\usepackage{floatflt}
\usepackage{wrapfig}
\usepackage[boxed, noend, noline]{algorithm2e}
\usepackage{subcaption}
\usepackage{extarrows}
\usepackage{pbox}
\usepackage{framed}
\usepackage{wrapfig}

\allowdisplaybreaks



\newcommand{\p}{{\bf p}}
\newcommand{\s}{{\bf s}}
\newcommand{\z}{{\bf z}}
\newcommand{\x}{{\bf x}}
\newcommand{\tx}{{\tilde{\x}}}
\newcommand{\ty}{{\tilde{y}}}
\newcommand{\Z}{{\mathcal{Z}}}
\newcommand{\F}{{\mathcal{F}}}
\newcommand{\D}{{\mathcal{D}}}
\newcommand{\0}{{\bf 0}}
\newcommand{\R}{\mathbb{R}}
\newcommand{\C}{\mathbb{C}}
\newcommand{\E}{\mathbb{E}}
\newcommand{\1}{{\bm 1}}
\newcommand{\Reg}{{\bf R}}
\newcommand{\sign}{{\text {sign}}}
\newcommand{\LHS}{{\text {LHS}}}
\newcommand{\RHS}{{\text {RHS}}}
\newcommand{\ANH}{{NormalHedge.DT}}
\newcommand{\ANB}{{NH-Boost.DT}}

\renewcommand{\P}{{\text {Pr}}}
\renewcommand{\H}{{\mathcal{H}}}
\renewcommand{\l}{{\bm \ell}}
\renewcommand{\(}{\left(}
\renewcommand{\)}{\right)}

\newtheorem{theorem}{Theorem}
\newtheorem{corollary}{Corollary}
\newtheorem{lemma}{Lemma}

\title{A Drifting-Games Analysis for Online Learning and Applications to Boosting}

\author{
Haipeng Luo \\
Department of Computer Science\\
Princeton University\\
Princeton, NJ 08540 \\
\texttt{haipengl@cs.princeton.edu} \\
\And
Robert E. Schapire\thanks{R. Schapire is currently at Microsoft Research in New York City.} \\
Department of Computer Science\\
Princeton University\\
Princeton, NJ 08540 \\
\texttt{schapire@cs.princeton.edu} \\
}

\nipsfinalcopy 

\begin{document}

\maketitle

\begin{abstract}
We provide a general mechanism to design online learning algorithms
based on a minimax analysis within a drifting-games framework.
Different online learning settings 
(Hedge, multi-armed bandit problems and online convex optimization)
are studied by converting into various kinds of drifting games.
The original minimax analysis for drifting games is then used
and generalized by applying a series of relaxations, 
starting from choosing a convex surrogate of the 0-1 loss function.
With different choices of surrogates, 
we not only recover existing algorithms, 
but also propose new algorithms that are totally parameter-free
and enjoy other useful properties.
Moreover, our drifting-games framework naturally allows us to study
high probability bounds without resorting to any concentration results,
and also a generalized notion of regret that measures how good the 
algorithm is compared to all but the top small fraction of candidates.
Finally, we translate our new Hedge algorithm into a new adaptive boosting 
algorithm that is computationally faster as shown in experiments, 
since it ignores a large number of examples on each round. 
\end{abstract}

\section{Introduction}
In this paper, we study online learning problems within a
drifting-games framework, 
with the aim of developing a general methodology for designing learning algorithms
based on a minimax analysis.

To solve an online learning problem,
it is natural to consider game-theoretically optimal algorithms 
which find the best solution even in worst-case scenarios.
This is possible for some special cases 
(\cite{CesabianchiFrHeHaScWa97, AbernethyBaRaTe08, AbernethyWa10, LuoSc14})
but difficult in general.
On the other hand, many other efficient algorithms with optimal regret rate 
(but not exactly minimax optimal) have been proposed for different learning settings 
(such as the exponential weights algorithm
\cite{FreundSc97, FreundSc99}, and follow the perturbed leader \cite{KalaiVe05}).
However, it is not always clear how to come up with these algorithms.
Recent work by Rakhlin et al. \cite{RakhlinShSr12}
built a bridge between these two classes of methods by showing that
many existing algorithms can indeed be derived from 
a minimax analysis followed by a series of relaxations.

In this paper, we provide a parallel way to design learning algorithms
by first converting online learning problems into variants of drifting games,
and then applying a minimax analysis and relaxations.
{\it Drifting games} \cite{Schapire01} (reviewed in Section \ref{sec:drifting})
generalize Freund's ``majority-vote game''
\cite{Freund95} and subsume some well-studied boosting and 
online learning settings. 
A nearly minimax optimal algorithm is proposed in \cite{Schapire01}.
It turns out the connections between drifting games and online learning
go far beyond what has been discussed previously.
To show that, we consider variants of drifting games that capture
different popular online learning problems.
We then generalize the minimax analysis in \cite{Schapire01}
based on one key idea: {\it relax a 0-1 loss function by a convex surrogate}.
Although this idea has been applied widely elsewhere in machine learning, 
we use it here in a new way to obtain a very general methodology 
for designing and analyzing online learning algorithms.
Using this general idea, we not only recover existing algorithms,
but also design new ones with special useful properties.
A somewhat surprising result is that our new algorithms are totally
{\it parameter-free}, which is usually not the case for
algorithms derived from a minimax analysis.
Moreover, a generalized notion of regret 
($\epsilon$-regret, defined in Section \ref{sec:hedge})
that measures how good the algorithm
is compared to all but the top $\epsilon$ fraction of candidates
arises naturally in our drifting-games framework.
Below we summarize our results for a range of learning settings.

\textbf{Hedge Settings:} (Section \ref{sec:hedge})
The Hedge problem \cite{FreundSc97} investigates how to cleverly bet across a set of actions.
We show an algorithmic equivalence between this problem and a simple 
drifting game (DGv1).
We then show how to relax the original minimax analysis step by step
to reach a general recipe for designing Hedge algorithms 
(Algorithm \ref{alg:hedge_recipe}).
Three examples of appropriate convex surrogates of the
0-1 loss function are then discussed, 
leading to the well-known exponential weights algorithm and two other new ones, 
one of which ({\ANH} in Section \ref{sec:3algs}) bears some similarities with the 
NormalHedge algorithm \cite{ChaudhuriFrHs09} 
and enjoys a similar $\epsilon$-regret bound {\it simultaneously} for all
$\epsilon$ and horizons. 
However, our regret bounds do not depend on the number of actions,
and thus can be applied even when there are infinitely many actions.
Our analysis is also arguably simpler and more intuitive than the one in 
\cite{ChaudhuriFrHs09} and easy to be generalized to more general settings.
Moreover, our algorithm is more computationally efficient since
it does not require a numerical searching step as in NormalHedge.
Finally, we also derive high probability bounds for the randomized Hedge
setting as a simple side product of our framework {\it without} using any
concentration results. 

\textbf{Multi-armed Bandit Problems:} (Section \ref{sec:gen})
The multi-armed bandit problem \cite{AuerCeFrSc02} is a classic example for learning with
incomplete information where the learner can only obtain feedback
for the actions taken.
To capture this problem, we study a quite different
drifting game (DGv2) where randomness and variance constraints
are taken into account.
Again the minimax analysis is generalized and the EXP3 algorithm
\cite{AuerCeFrSc02} is recovered. 
Our results could be seen as a preliminary step to answer the
open question \cite{AbernethyWa09} on exact minimax optimal algorithms 
for the multi-armed bandit problem. 

\textbf{Online Convex Optimization:} (Section \ref{sec:gen})
Based the theory of convex optimization,
online convex optimization \cite{Zinkevich03} has been the foundation of modern
online learning theory.
The corresponding drifting game formulation is a continuous
space variant (DGv3).
Fortunately, it turns out that all results from the Hedge setting
are ready to be used here, 
recovering the continuous EXP algorithm \cite{Cover91, HazanAgKa07, NarayananRa10}
and also generalizing our new algorithms to this general setting. 
Besides the usual regret bounds, we also generalize the $\epsilon$-regret,
which, as far as we know, is the first time it has been explicitly studied.
Again, we emphasize that our new algorithms are adaptive 
in $\epsilon$ and the horizon.

\textbf{Boosting:} (Section \ref{sec:gen})
Realizing that every Hedge algorithm can be converted into 
a boosting algorithm (\cite{SchapireFr12}), 
we propose a new boosting algorithm ({\ANB}) by converting {\ANH}.
The adaptivity of {\ANH} is then translated into training error and
margin distribution bounds that previous analysis in \cite{SchapireFr12} 
using nonadaptive algorithms does not show.
Moreover, our new boosting algorithm 
ignores a great many examples on each round, which is
an appealing property useful to speeding up the weak learning algorithm.
This is confirmed by our experiments.


\textbf{Related work}:
Our analysis makes use of potential functions. 
Similar concepts have widely appeared in the literature
\cite{CesabianchiLu03, AudibertBuLu14},
but unlike our work, they are not related to any minimax analysis
and might be hard to interpret. 
The existence of parameter free Hedge algorithms for  
unknown number of actions was shown in \cite{ChernovVovk10},
but no concrete algorithms were given there.
Boosting algorithms that ignore some examples on each round
were studied in \cite{FriedmanHaTi00}, where a heuristic was
used to ignore examples with small weights and no theoretical
guarantee is provided.

\section{Reviewing Drifting Games}
\label{sec:drifting}
We consider a simplified version of drifting games similar to the one described in 
\citep[chap. 13]{SchapireFr12} (also called chip games).
This game proceeds through $T$ rounds, 
and is played between a player and an adversary who controls $N$ chips on the real line.
The positions of these chips at the end of round $t$ are denoted by $\s_t \in \R^N$,
with each coordinate $s_{t,i}$ corresponding to the position of chip $i$.
Initially, all chips are at position $0$ so that $\s_0 = \0$.
On every round $t = 1, \ldots, T$: 
the player first chooses a distribution $\p_t$ over the chips,
then the adversary decides the movements of the chips
$\z_t$ so that the new positions are updated as $\s_t = \s_{t-1} + \z_t$.
Here, each $z_{t,i}$ has to be picked from a prespecified set $B \subset \R$,
and more importantly, satisfy the constraint $\p_t \cdot \z_t \geq \beta \geq 0$
for some fixed constant $\beta$.

At the end of the game, each chip is associated with a nonnegative loss
defined by $L(s_{T,i})$ for some nonincreasing function
$L$ mapping from the final position of the chip to $\R_+$. 
The goal of the player is to minimize the chips' average loss 
$ \frac{1}{N} \sum_{i=1}^N L(s_{T,i})$ after $T$ rounds.
So intuitively, the player aims to ``push'' the chips to the right by
assigning appropriate weights on them so that the adversary has to move them 
to the right by $\beta$ in a weighted average sense on each round.
This game captures many learning problems. 
For instance, binary classification via boosting can be translated into
a drifting game by treating each training example as a chip 
(see \cite{Schapire01} for details).

We regard a player's strategy $\D$ as a function
mapping from the history of the adversary's decisions to a distribution 
that the player is going to play with, 
that is, $\p_t = \D(\z_{1:t-1})$ where $\z_{1:t-1}$ stands for $\z_1, \ldots, \z_{t-1}$. 
The player's worst case loss using this algorithm is then denoted by $L_T(\D)$.
The minimax optimal loss of the game is computed by the following expression:
$ \min_{\D}L_T(\D) = \min_{\p_1\in \Delta_N}\max_{\z_1 \in \Z_{\p_1}} \cdots  
\min_{\p_T\in \Delta_N}\max_{\z_T \in \Z_{\p_T}}  
\frac{1}{N}\sum_{i=1}^N L(\sum_{t=1}^T z_{t,i}) ,$
where $\Delta_N$ is the $N$ dimensional simplex and 
$\Z_\p = B^N \cap \{\z: \p\cdot\z \geq \beta\}$ is assumed to be compact. 
A strategy $\D^*$ that realizes the minimum in $\min_{\D}L_T(\D)$
is called a minimax optimal strategy. 
A nearly optimal strategy and its analysis is originally given in \cite{Schapire01},
and a derivation by directly tackling the above minimax expression can
be found in \citep[chap. 13]{SchapireFr12}.
Specifically, a sequence of potential functions of a chip's position 
is defined recursively as follows:
\begin{equation}\label{equ:minimax_potentials}
\Phi_T(s) = L(s), \quad \Phi_{t-1}(s) = \min_{w \in \R_+}\max_{z\in B} (\Phi_t(s+z) + w(z - \beta)).
\end{equation}
Let $w_{t,i}$ be the weight that realizes the minimum in the definition of 
$\Phi_{t-1}(s_{t-1,i})$, that is, 
$w_{t,i} \in \arg\min_w\max_z(\Phi_t(s_{t-1,i}+z)+w(z-\beta))$.
Then the player's strategy is to set $p_{t,i} \propto w_{t,i}$.
The key property of this strategy is that it assures that the sum
of the potentials over all the chips never increases, 
connecting the player's final loss with the potential at time $0$ as follows:
\begin{equation}\label{equ:monotonic_potentials}
\frac{1}{N}\sum_{i=1}^N L(s_{T,i}) \leq \frac{1}{N}\sum_{i=1}^N \Phi_T(s_{T,i})
\leq \frac{1}{N}\sum_{i=1}^N \Phi_{T-1}(s_{T-1,i}) \leq  \cdots 
\leq \frac{1}{N}\sum_{i=1}^N \Phi_{0}(s_{0,i})  = \Phi_{0}(0) .
\end{equation}
It has been shown in \cite{Schapire01} that this upper bound on the loss
is optimal in a very strong sense.

Moreover, in some cases the potential functions have nice closed forms and thus the
algorithm can be efficiently implemented. 
For example, in the boosting setting, $B$ is simply $\{-1,+1\}$, 
and one can verify 
$ \Phi_t(s) = \frac{1+\beta}{2}\Phi_{t+1}(s+1) + \frac{1-\beta}{2}\Phi_{t+1}(s-1)$ and 
$ w_{t,i} = \tfrac{1}{2}\(\Phi_t(s_{t-1,i} - 1) - \Phi_t(s_{t-1,i} + 1)\)$.
With the loss function $L(s)$ being $\1\{s\leq 0\}$, these can be further simplified
and eventually give exactly the boost-by-majority algorithm \cite{Freund95}.

\section{Online Learning as a Drifting Game}
\label{sec:hedge}
The connection between drifting games and some specific settings of online learning
has been noticed before (\cite{Schapire01, MukherjeeSc10}).
We aim to find deeper connections or even an equivalence between variants 
of drifting games and more general settings of online learning, 
and provide insights on designing learning algorithms through
a minimax analysis.
We start with a simple yet classic Hedge setting. 

\subsection{Algorithmic Equivalence}
In the Hedge setting \cite{FreundSc97}, a player tries to earn as much as possible
(or lose as little as possible) by cleverly spreading a fixed amount 
of money to bet on a set of actions on each day.   
Formally, the game proceeds for $T$ rounds, and
on each round $t = 1, \ldots, T$: 
the player chooses a distribution $\p_t$ over $N$ actions,
then the adversary decides the actions' losses  $\l_t$
(i.e. action $i$ incurs loss $\ell_{t,i}\in [0,1]$) which are revealed to the player.
The player suffers a weighted average loss $\p_t\cdot\l_t$ at the end of this round.
The goal of the player is to minimize his ``regret'', which is
usually defined as the difference between his total loss and the loss of the best action.
Here, we consider an even more general notion of regret studied 
in \cite{Kleinberg05, Kleinberg06, ChaudhuriFrHs09, ChernovVovk10}, which we call {\it $\epsilon$-regret}.
Suppose the actions are ordered according to their total losses after $T$ rounds
(i.e. $\sum_{t=1}^T \ell_{t,i}$) from smallest to largest,
and let $i_\epsilon$ be the index of the action that is the 
$\lceil N\epsilon \rceil$-th element in the sorted list ($0 < \epsilon \leq 1$).
Now, $\epsilon$-regret is defined as
$ \Reg_{T}^\epsilon(\p_{1:T}, \l_{1:T}) 
= \sum_{t=1}^T \p_t\cdot \l_t - \sum_{t=1}^T \ell_{t, i_\epsilon}. $
In other words, $\epsilon$-regret measures the difference between
the player's loss and the loss of the $\lceil N\epsilon \rceil$-th best action
(recovering the usual regret with $\epsilon \leq 1/N$),
and sublinear $\epsilon$-regret implies that the player's loss is almost 
as good as all but the top $\epsilon$ fraction of actions.
Similarly, 
$\Reg_{T}^\epsilon(\H)$
denotes the worst case $\epsilon$-regret for a specific algorithm $\H$.
For convenience, when $\epsilon \leq 0$ or $\epsilon > 1$,
we define $\epsilon$-regret to be $\infty$ or $-\infty$ respectively.

Next we discuss how Hedge is highly related to drifting games.
Consider a variant of drifting games where $B = [-1, 1], \beta=0$
and $L(s) = \1\{s \leq -R\}$ for some constant $R$. 
Additionally, we impose an extra restriction on the adversary: 
$|z_{t, i} - z_{t, j}| \leq 1$ for all $i$ and $j$.
In other words, the difference between any two chips' movements is at most $1$. 
We denote this specific variant of drifting games by DGv1 (summarized in
 Appendix \ref{apd:DGv}) 
and a corresponding algorithm by $\D_R$ to emphasize the dependence on $R$.
The reductions in Algorithm \ref{alg:hedge2drift} and \ref{alg:drift2hedge} and 
Theorem \ref{thm:equivalence} show that DGv1 and the Hedge problem are algorithmically equivalent
(note that both conversions are valid). 
The proof is straightforward and deferred to Appendix \ref{apd:equiv}.
By Theorem \ref{thm:equivalence}, it is clear that the minimax optimal 
algorithm for one setting is also minimax optimal for the other 
under these conversions.

\begin{figure}[t]
\centering
\SetAlCapSkip{.2em}
\begin{minipage}{0.49\textwidth}
\IncMargin{.5em}
\begin{algorithm}[H]
\caption{Conversion of a Hedge Algorithm $\H$ to a DGv1 Algorithm $\D_R$}
\label{alg:hedge2drift}

\SetKwInOut{Input}{Input}
\Input{A Hedge Algorithm $\mathcal{H}$}

\For{ $t = 1$ \KwTo $T$} {
Query $\H$: $\p_t = \H(\l_{1:t-1})$. \\
Set:  $\D_R(\z_{1:t-1}) = \p_t$. \\
Receive movements $\z_t$ from the adversary. \\
Set:  $\ell_{t,i} = z_{t,i} - \min_j z_{t,j}, \;\forall i$. \\
}
\end{algorithm}
\DecMargin{.5em}
\end{minipage}
\begin{minipage}{0.49\textwidth}
\IncMargin{.5em}
\begin{algorithm}[H]
\caption{Conversion of a DGv1 Algorithm $\D_R$ to a Hedge Algorithm $\H$}
\label{alg:drift2hedge}

\SetKwInOut{Input}{Input}
\Input{A DGv1 Algorithm $\D_R$}

\For{ $t = 1$ \KwTo $T$} {
Query $\D_R$: $\p_t = \D_R(\z_{1:t-1})$. \\
Set:  $\H(\l_{1:t-1}) = \p_t$. \\
Receive losses $\l_t$ from the adversary. \\
Set:  $z_{t,i} = \ell_{t,i} - \p_t \cdot \l_t, \;\forall i$. \\
}
\end{algorithm}
\DecMargin{.5em}
\end{minipage}
\end{figure}

\begin{theorem}\label{thm:equivalence}
DGv1 and the Hedge problem are algorithmically equivalent in the following sense: \\
(1) Algorithm \ref{alg:hedge2drift} produces a DGv1 algorithm $\D_R$
satisfying $L_T(\D_R) \leq i/N$ where $i\in\{0,\ldots,N\}$ is such that 
$\Reg_T^{(i+1)/N}(\H) < R \leq \Reg_T^{i/N}(\H)$. 

(2) Algorithm \ref{alg:drift2hedge} produces a Hedge algorithm $\H$
with $\Reg_T^{\epsilon}(\H) < R$ for any $R$ such that $L_T(\D_R) < \epsilon$. 
\end{theorem}

\subsection{Relaxations}
From now on we only focus on the direction of converting a drifting game algorithm 
into a Hedge algorithm.
In order to derive a minimax Hedge algorithm,
Theorem \ref{thm:equivalence} tells us it suffices to derive minimax
DGv1 algorithms.
Exact minimax analysis is usually difficult,
and appropriate relaxations seem to be necessary.
To make use of the existing analysis for standard drifting games,
the first obvious relaxation is to drop the additional restriction 
in DGv1, that is, $|z_{t,i}-z_{t,j}| \leq 1$ for all $i$ and $j$.
Doing this will lead to the exact setting discussed in \cite{MukherjeeSc10}
where a near optimal strategy is proposed using the recipe in 
Eq. \eqref{equ:minimax_potentials}.
It turns out that this relaxation is reasonable and does not 
give too much more power to the adversary.
To see this, first recall that results from \cite{MukherjeeSc10}, 
written in our notation, state that 
$ \min_{\D_R} L_T(\D_R) \leq \frac{1}{2^T}\sum_{j=0}^{\frac{T-R}{2}} 
\binom{T+1}{j}, $
which, by Hoeffding's inequality, is upper bounded by
$ 2\exp\(-\frac{(R+1)^2}{2(T+1)}\) $.
Second, statement (2) in Theorem \ref{thm:equivalence}
clearly remains valid if the input of Algorithm \ref{alg:drift2hedge}
is a drifting game algorithm for this relaxed version of DGv1.
Therefore, by setting $\epsilon > 2\exp\(-\frac{(R+1)^2}{2(T+1)}\)$
and solving for $R$,
we have $\Reg_T^{\epsilon}(\H) \leq O\(\sqrt{T\ln (\frac{1}{\epsilon}})\)$,
which is the known optimal regret rate for the Hedge problem,
showing that we lose little due to this relaxation.

However, the algorithm proposed in \cite{MukherjeeSc10} is not
computationally efficient since the potential functions $\Phi_t(s)$
do not have closed forms.
To get around this, we would want the minimax expression
in Eq. \eqref{equ:minimax_potentials} to be easily solved,
just like the case when $B=\{-1,1\}$.
It turns out that convexity would allow us to treat $B=[-1,1]$ almost
as $B=\{-1,1\}$.
Specifically, if each $\Phi_t(s)$ is a convex function of $s$,
then due to the fact that the maximum of a convex function
is always realized at the boundary of a compact region, we have
\begin{equation}\label{equ:convexity}
\min_{w \in \R_+}\max_{z\in [-1,1]} \(\Phi_t(s+z) + wz\) 
= \min_{w \in \R_+}\max_{z\in \{-1,1\}} \(\Phi_t(s+z) + wz\) 
= \frac{\Phi_t(s-1) + \Phi_t(s+1)}{2} ,
\end{equation}
with $w = (\Phi_t(s-1)-\Phi_t(s+1))/2$ realizing the minimum.
Since the 0-1 loss function $L(s)$ is not convex,
this motivates us to find a convex surrogate of $L(s)$.
Fortunately, relaxing the equality constraints in 
Eq. \eqref{equ:minimax_potentials} does not affect
the key property of Eq. \eqref{equ:monotonic_potentials}
as we will show in the proof of Theorem \ref{thm:hedge_recipe}.
``Compiling out'' the input of Algorithm \ref{alg:drift2hedge},
we thus have our general recipe (Algorithm \ref{alg:hedge_recipe}) for 
designing Hedge algorithms with the following regret guarantee.

\begin{figure}[t]
\centering
\SetAlCapSkip{.5em}
\IncMargin{.5em}
\begin{algorithm}[H]
\caption{A General Hedge Algorithm $\H$}
\label{alg:hedge_recipe}

\SetKwInOut{Input}{Input}
\SetKw{DownTo}{down to}
\Input{A convex, nonincreasing, nonnegative function $\Phi_T(s)$.
}

\For{ $t = T$ \DownTo $1$} {
Find a convex function $\Phi_{t-1}(s)$ s.t.  $\forall s,$
$\Phi_{t}(s-1)+\Phi_{t}(s+1) \leq 2\Phi_{t-1}(s).$
}
Set: $\s_0 = \0$. \\
\For{ $t = 1$ \KwTo $T$} {
Set:  $\H(\l_{1:t-1}) = \p_t$ s.t. $p_{t,i} \propto 
\Phi_t(s_{t-1,i}-1) - \Phi_t(s_{t-1,i}+1)$. \\
Receive losses $\l_t$ and set $s_{t,i} = s_{t-1,i} + \ell_{t,i} - \p_t \cdot \l_t, \;\forall i$. \\
}
\end{algorithm}
\DecMargin{.5em}
\end{figure}

\begin{theorem}\label{thm:hedge_recipe}
For Algorithm \ref{alg:hedge_recipe},
if $R$ and $\epsilon$ are such that $\Phi_0(0) < \epsilon$
and $\Phi_T(s) \geq \1\{s \leq -R\}$ for all $s\in\R$,
then $\Reg_T^{\epsilon}(\H) < R$.
\end{theorem}
{\it Proof.} It suffices to show that Eq. \eqref{equ:monotonic_potentials} holds so
that the theorem follows by a direct application of statement (2) of
Theorem \ref{thm:equivalence}.
Let $w_{t,i} = (\Phi_t(s_{t-1,i}-1)-\Phi_t(s_{t-1,i}+1))/2$.
Then $\sum_i \Phi_t(s_{t,i}) \leq 
\sum_i \(\Phi_t(s_{t-1,i}+z_{t,i}) + w_{t,i}z_{t,i}\)$
since $p_{t,i} \propto w_{t,i}$ and $\p_t\cdot\z_t \geq 0$.
On the other hand, by Eq. \eqref{equ:convexity}, we have
$ \Phi_t(s_{t-1,i}+z_{t,i}) + w_{t,i}z_{t,i}
\leq \min_{w\in \R_+}\max_{z\in [-1,1]}  \(\Phi_t(s_{t-1,i}+z) + wz\) 
= \frac{1}{2}\(\Phi_t(s_{t-1,i}-1)+\Phi_t(s_{t-1,i}+1)\) $,
which is at most $\Phi_{t-1}(s_{t-1,i})$ by Algorithm \ref{alg:hedge_recipe}.
This shows $\sum_i \Phi_t(s_{t,i}) \leq \sum_i \Phi_{t-1}(s_{t-1,i})$ and 
Eq. \eqref{equ:monotonic_potentials} follows.  \qed

Theorem 2 tells us that if solving $\Phi_0(0) < \epsilon$ for $R$ gives $R > \underline{R}$
for some value $\underline{R}$, then the regret of Algorithm \ref{alg:hedge_recipe}
is less than any value that is greater than $\underline{R}$, meaning the regret is
at most $\underline{R}$.

\subsection{Designing Potentials and Algorithms}\label{sec:3algs}
Now we are ready to recover existing algorithms and develop new
ones by choosing an appropriate potential $\Phi_T(s)$ as Algorithm 
\ref{alg:hedge_recipe} suggests. 
We will discuss three different algorithms below,
and summarize these examples in Table \ref{tab:algs} 
(see Appendix \ref{apd:lemmas}).

\paragraph{Exponential Weights (EXP) Algorithm.}
Exponential loss is an obvious choice for $\Phi_T(s)$ as it has
been widely used as the convex surrogate of the 0-1 loss 
function in the literature.
It turns out that this will lead to the well-known 
exponential weights algorithm \cite{FreundSc97, FreundSc99}.
Specifically, we pick $\Phi_T(s)$ to be $\exp\(-\eta(s+R)\)$
which exactly upper bounds $\1\{s \leq -R\}$.
To compute $\Phi_t(s)$ for $t \leq T$, 
we simply let $\Phi_{t}(s-1)+\Phi_{t}(s+1) \leq 2\Phi_{t-1}(s)$ hold
with equality.
Indeed, direct computations show that all $\Phi_t(s)$ share a similar form: 
$\Phi_t(s) = \(\frac{e^\eta+e^{-\eta}}{2}\)^{T-t} \cdot \exp\(-\eta(s+R)\).$
Therefore, according to Algorithm \ref{alg:hedge_recipe}, 
the player's strategy is to set
$$ p_{t,i} \propto  \Phi_t(s_{t-1,i}-1) - \Phi_t(s_{t-1,i}+1) 
\propto  \exp\(-\eta s_{t-1,i}\),$$
which is exactly the same as EXP (note that $R$ becomes irrelevant after normalization).
To derive regret bounds, it suffices to require
$\Phi_0(0) 
< \epsilon,$
which is equivalent to 
$R > \frac{1}{\eta}\(\ln(\frac{1}{\epsilon}) + T\ln\frac{e^\eta+e^{-\eta}}{2}\).$
By Theorem \ref{thm:hedge_recipe} and Hoeffding's lemma
(see \citep[Lemma A.1]{CesabianchiLu06}), 
we thus know
$ \Reg_T^{\epsilon}(\H) \leq  \frac{1}{\eta}\ln\(\frac{1}{\epsilon}\) + \frac{T\eta}{2}
= \sqrt{2T\ln\(\frac{1}{\epsilon}\)}$
where the last step is by optimally tuning $\eta$ to be 
$\sqrt{2(\ln\frac{1}{\epsilon})/T}$.
Note that this algorithm is {\it not adaptive} in the sense that
it requires knowledge of $T$ and $\epsilon$ 
to set the parameter $\eta$.

We have thus recovered the well-known EXP
algorithm and given a new analysis using the drifting-games framework.
More importantly, as in \cite{RakhlinShSr12}, 
this derivation may shed light on why this algorithm works
and where it comes from, namely,
a minimax analysis followed by a series of relaxations,
starting from a reasonable surrogate of the 0-1 loss function.

\paragraph{2-norm Algorithm.}
We next move on to another simple convex surrogate:
$ \Phi_T(s) = a[s]_-^2 \geq \1\{s \leq -1/\sqrt{a}\}, $
where $a$ is some positive constant and $[s]_- = \min\{0, s\}$
represents a truncating operation.
The following lemma shows that $\Phi_t(s)$ can also be simply described.

\begin{lemma}\label{lem:2-norm}
If $a > 0$, then $\Phi_t(s) = a\([s]_-^2 + T- t\)$ satisfies 
$\Phi_{t}(s-1)+\Phi_{t}(s+1) \leq 2\Phi_{t-1}(s)$. 
\end{lemma}

Thus, Algorithm 3 can again be applied.  
The resulting algorithm is extremely concise:
$$ p_{t,i} \propto \Phi_t(s_{t-1,i}-1) - \Phi_t(s_{t-1,i}+1) 
\propto [s_{t-1,i}-1]_-^2 - [s_{t-1,i}+1]_-^2.$$
We call this the ``2-norm'' algorithm since it resembles
the $p$-norm algorithm in the literature when $p=2$
(see \cite{CesabianchiLu06}).
The difference is that the $p$-norm algorithm sets the weights proportional to
the derivative of potentials, instead of the difference of them as we are doing here.
A somewhat surprising property of this algorithm is that
it is totally adaptive and parameter-free (since $a$ disappears under normalization),
a property that we usually do not expect to obtain
from a minimax analysis.
Direct application of Theorem \ref{thm:hedge_recipe} 
($\Phi_0(0) = aT < \epsilon \Leftrightarrow 1/\sqrt{a} > \sqrt{T/\epsilon}$)
shows that its regret achieves the optimal dependence on the horizon $T$.

\begin{corollary}\label{cor:2-norm}
Algorithm \ref{alg:hedge_recipe} with potential $\Phi_t(s)$ defined in
Lemma \ref{lem:2-norm} produces a Hedge algorithm $\H$ such
that $\Reg_T^\epsilon(\H) \leq \sqrt{T/\epsilon}$ simultaneously
for all $T$ and $\epsilon$.
\end{corollary}

\paragraph{\ANH.}
The regret for the 2-norm algorithm does not have the optimal
dependence on $\epsilon$. 
An obvious follow-up question would be
whether it is possible to derive an adaptive algorithm that achieves
the optimal rate $O(\sqrt{T\ln (1/\epsilon)})$ simultaneously
for all $T$ and $\epsilon$ using our framework.
An even deeper question is: instead of choosing convex surrogates
in a seemingly arbitrary way, 
is there a more natural way to find the {\it right} choice of $\Phi_T(s)$?

To answer these questions, we recall that the reason why
the 2-norm algorithm can get rid of the dependence on $\epsilon$
is that $\epsilon$ appears merely in the multiplicative constant
$a$ that does not play a role after normalization.
This motivates us to let $\Phi_T(s)$ in the form of $\epsilon F(s)$
for some $F(s)$.
On the other hand,  from Theorem \ref{thm:hedge_recipe},
we also want $\epsilon F(s)$ to upper bound the 0-1 loss function
$\1\{s \leq -\sqrt{dT\ln(1/\epsilon)} \}$ for some constant $d$.
Taken together, this is telling us that the right choice of $F(s)$ should be 
of the form $\Theta\(\exp(s^2/T)\)$\footnote{
Similar potential was also proposed in recent work \cite{McmahanOr14, Orabona14}
for a different setting.
}.
Of course we still need to refine it to satisfy the monotonicity 
and other properties.
We define $\Phi_T(s)$ formally and more generally as:
$$ \Phi_T(s) = a\(\exp\(\tfrac{[s]_-^2}{dT}\) - 1\) 
\geq \1\left\{ s \leq -\sqrt{dT \ln \(\tfrac{1}{a} + 1\)} \right\}, $$
where $a$ and $d$ are some positive constants.
This time it is more involved to figure out what other $\Phi_t(s)$ 
should be. 
The following lemma addresses this issue (proof deferred to Appendix \ref{apd:lemmas}).

\begin{lemma}\label{lem:anh}
If $ b_t = 1 - \frac{1}{2}\sum_{\tau = t+1}^T \(\exp\(\frac{4}{d\tau}\) - 1\), a > 0, d \geq 3 $ and
$\Phi_t(s) = a\(\exp\(\frac{[s]_-^2}{dt}\) - b_t\)$ (define $\Phi_0(s) \equiv a(1-b_0)$),
then we have $\Phi_{t}(s-1)+\Phi_{t}(s+1) \leq 2\Phi_{t-1}(s)$
for all $s \in \R$ and $t = 2, \ldots, T$. 
Moreover, Eq. \eqref{equ:monotonic_potentials} still holds.
\end{lemma}

Note that even if $\Phi_{1}(s-1)+\Phi_{1}(s+1) \leq 2\Phi_{0}(s)$ is not valid in general,
Lemma \ref{lem:anh} states that Eq. \eqref{equ:monotonic_potentials} still holds.
Thus Algorithm \ref{alg:hedge_recipe} can indeed still be applied,
leading to our new algorithm:
$$ p_{t,i} \propto \Phi_t(s_{t-1,i}-1) - \Phi_t(s_{t-1,i}+1) 
\propto \exp\(\tfrac{[s_{t-1,i}-1]_-^2}{dt}\) - 
\exp\(\tfrac{[s_{t-1,i}+1]_-^2}{dt}\) .$$
Here, $d$ seems to be an extra parameter,
but in fact, simply setting $d=3$ is good enough:

\begin{corollary}\label{cor:anh}
Algorithm \ref{alg:hedge_recipe} with potential $\Phi_t(s)$ defined in
Lemma \ref{lem:anh} and $d=3$ produces a Hedge algorithm $\H$ such
that the following holds simultaneously for all $T$ and $\epsilon$:
$$\Reg_T^\epsilon(\H) \leq \sqrt{3T\ln \(\tfrac{1}{2\epsilon}
\(e^{4/3}-1\)\(\ln T + 1\)+1\)} = O\(\sqrt{T\ln\(1/\epsilon\) + T\ln\ln T}\).
$$ 
\end{corollary}

We have thus proposed a parameter-free adaptive algorithm
with optimal regret rate (ignoring the $\ln\ln T$ term) 
using our drifting-games framework.
In fact, our algorithm bears a striking similarity to 
NormalHedge \cite{ChaudhuriFrHs09}, 
the first algorithm that has this kind of adaptivity.
We thus name our algorithm {\ANH}\footnote{``DT'' stands for discrete time.}.
We include NormalHedge in Table \ref{tab:algs} for comparison.
One can see that the main differences are: 
1) On each round NormalHedge performs a numerical search to 
find out the right parameter used in the exponents; 
2) NormalHedge uses the derivative of potentials as weights.

Compared to NormalHedge, 
the regret bound for {\ANH} has no explicit dependence on $N$, 
but has a slightly worse dependence on $T$  (indeed $\ln\ln T$ is almost negligible).
We emphasize other advantages of our algorithm over NormalHedge:
1) {\ANH} is more computationally efficient especially when $N$ is very large,
since it does not need a numerical search for each round; 
2) our analysis is arguably simpler and more intuitive than the one in 
\cite{ChaudhuriFrHs09};
3) as we will discuss in Section \ref{sec:gen}, {\ANH}
can be easily extended to deal with the more general online 
convex optimization problem where the number of actions is infinitely large,
while it is not clear how to do that for NormalHedge by generalizing the analysis in \cite{ChaudhuriFrHs09}.
Indeed, the extra dependence on the number of actions $N$ 
for the regret of NormalHedge makes 
this generalization even seem impossible.
Finally, we will later see that {\ANH} outperforms NormalHedge in experiments.
Despite the differences, it is worth noting that
both algorithms assign zero weight to some actions on each round,
an appealing property when $N$ is huge. 
We will discuss more on this in Section \ref{sec:gen}.

\subsection{High Probability Bounds}
We now consider a common variant of Hedge:
on each round, instead of choosing a distribution $\p_t$, 
the player has to randomly pick a single action $i_t$, 
while the adversary decides the losses $\l_t$ 
at the same time (without seeing $i_t$).
For now we only focus on the player's regret to the best action:
$ \Reg_T(i_{1:T}, \l_{1:T}) =  \sum_{t=1}^T \ell_{t, i_t} - 
\min_i\sum_{t=1}^T \ell_{t, i}. $
Notice that the regret is now a random variable,
and we are interested in a bound that holds with high probability.
Using Azuma's inequality, 
standard analysis (see for instance \citep[Lemma 4.1]{CesabianchiLu06})
shows that the player can simply draw $i_t$
according to $\p_t = \H(\l_{1:t-1})$, 
the output of a standard Hedge algorithm,
and suffers regret at most $\Reg_T(\H) + \sqrt{T\ln( 1/\delta)}$
with probability $1-\delta$.
Below we recover similar results as a simple side product
of our drifting-games analysis {\it without} resorting to concentration results,
such as Azuma's inequality.

For this, we only need to modify Algorithm \ref{alg:hedge_recipe}
by setting $z_{t,i} = \ell_{t,i} - \ell_{t, i_t}$.
The restriction $\p_t\cdot \z_t \geq 0$ is then
relaxed to hold in expectation. 
Moreover, it is clear that Eq. \eqref{equ:monotonic_potentials} also still holds in expectation.
On the other hand, 
by definition and the union bound, one can show that
$\sum_i \E[L(s_{T,i})] = \sum_i \P\left[ s_{T,i} \leq -R \right] \geq 
\P\left[ \Reg_T(i_{1:T}, \l_{1:T}) \geq R \right]$.
So setting $\Phi_0(0) = \delta$
shows that the regret is smaller than $R$ with probability $1-\delta$.
Therefore, for example, if EXP is used, then the regret would be at most 
$\sqrt{2T\ln(N/\delta)}$ with probability $1-\delta$,
giving basically the same bound as the standard analysis.
One draw back is that EXP would need $\delta$
as a parameter. 
However, this can again be addressed by {\ANH} for the 
exact same reason that {\ANH} is independent of $\epsilon$.
We have thus derived high probability bounds without using
any concentration inequalities.

\section{Generalizations and Applications}
\label{sec:gen}
\textbf{Multi-armed Bandit (MAB) Problem:}
The only difference between Hedge (randomized version) and 
the non-stochastic MAB problem \cite{AuerCeFrSc02}
is that on each round, after picking $i_t$, the player only sees the loss 
for this single action $\ell_{t,i_t}$ instead of the whole vector $\l_t$.
The goal is still to compete with the best action.
A common technique used in the bandit setting is to 
build an unbiased estimator $\hat\l_t$ for the losses, 
which in this case could be 
$\hat\ell_{t,i} = \1\{i = i_t\}\cdot\ell_{t,i_t}/p_{t, i_t}$.
Then algorithms such as EXP can be used by replacing $\l_t$ with $\hat\l_t$,
leading to the EXP3 algorithm \cite{AuerCeFrSc02} 
with regret $O(\sqrt{TN\ln N})$.

One might expect that Algorithm \ref{alg:hedge_recipe} would also work well
by replacing $\l_t$ with $\hat\l_t$.
However, doing so breaks an important
property of the movements $z_{t,i}$: boundedness.
Indeed, Eq. \eqref{equ:convexity} no longer makes sense
if $z$ could be infinitely large,
even if in expectation it is still in $[-1,1]$ 
(note that $z_{t,i}$ is now a random variable).
It turns out that we can address this issue by imposing 
a variance constraint on $z_{t,i}$.
Formally, we consider a variant of drifting games
where on each round, the adversary 
picks a random movement
$z_{t,i}$ for each chip such that: $z_{t,i} \geq -1, \E_t[z_{t,i}] \leq 1,
\E_t[z_{t,i}^2] \leq 1/p_{t,i}$ and $\E_t[\p_t \cdot \z_t] \geq 0$. 
We call this variant DGv2 and summarize it in Appendix \ref{apd:DGv}.
The standard minimax analysis and the derivation of potential
functions need to be modified in a certain way for DGv2,
as stated in Theorem \ref{thm:DGv2} (Appendix \ref{apd:bandit}).
Using the analysis for DGv2, we propose a general recipe for designing 
MAB algorithms in a similar way as for Hedge and also recover
EXP3 (see Algorithm \ref{alg:bandit_recipe} and 
Theorem \ref{thm:bandit_recipe} in Appendix \ref{apd:bandit}).
Unfortunately so far we do not know other appropriate potentials
due to some technical difficulties. 
We conjecture, however, that there is a potential function that could 
recover the poly-INF algorithm \cite{AudibertBu10, AudibertBuLu14}
or give its variants that achieve the optimal regret $O(\sqrt{TN})$.

\textbf{Online Convex Optimization:}
We next consider a general online
convex optimization setting \cite{Zinkevich03}.
Let $S \subset \R^d$ be a compact convex set, and $\F$ 
be a set of convex functions with range $[0,1]$ on $S$.
On each round $t$, the learner chooses a point $\x_t \in S$, 
and the adversary chooses a loss function $f_t \in \F$ (knowing $\x_t$). 
The learner then suffers loss $f_t(\x_t)$. 
The regret after $T$ rounds is 
$ \Reg_T(\x_{1:T}, f_{1:T}) = \sum_{t=1}^T f_t(\x_t) - \min_{\x \in S} \sum_{t=1}^T f_t(\x)$.
There are two general approaches to OCO:
one builds on convex optimization theory \cite{Shalevshwartz11},
and the other generalizes EXP to a continuous space \cite{Cover91, NarayananRa10}.
We will see how the drifting-games framework 
can recover the latter method and also leads to new ones.

To do so,  we introduce a continuous variant of drifting games
(DGv3, see Appendix \ref{apd:DGv}).
There are now infinitely many chips, one for each point in $S$.
On round $t$, the player needs to choose a distribution over the chips,
that is, a probability density function $p_t(\x)$ on $S$. 
Then the adversary decides the movements for each chip, that is,
a function $z_t(\x)$ with range $[-1,1]$ on $S$
(not necessarily convex or continuous),
subject to a constraint $\E_{\x \sim p_t} [z_t(\x)] \geq 0$.
At the end, each point $\x$ is associated with a loss
$ L(\x) = \1\{\sum_t z_t(\x) \leq -R \}$,
and the player aims to minimize the total loss $\int_{\x\in S} L(\x) d\x$.

OCO can be converted into DGv3 by setting
$z_t(\x) = f_t(\x) - f_t(\x_t)$ and predicting $\x_t = \E_{\x \sim p_t} [\x] \in S$.
The constraint $\E_{\x \sim p_t} [z_t(\x)] \geq 0$ holds
by the convexity of $f_t$.
Moreover, it turns out that the minimax analysis and potentials
for DGv1 can readily be used here,
and the notion of $\epsilon$-regret, now generalized to the OCO setting, 
measures the difference of the player's loss and the loss of 
a best fixed point in a subset of $S$ that excludes the top $\epsilon$
fraction of points. 
With different potentials,
we obtain versions of each of the three algorithms of Section \ref{sec:hedge}
generalized to this setting, 
with the same $\epsilon$-regret bounds as before.
Again, two of these methods are adaptive and parameter-free.
To derive bounds for the usual regret, 
at first glance it seems that we have to set $\epsilon$ to be close to zero,
leading to a meaningless bound. 
Nevertheless, this is addressed by Theorem \ref{thm:OCO_recipe}
using similar techniques in \cite{HazanAgKa07}, 
giving the usual $O(\sqrt{dT\ln T})$ regret bound. 
All details can be found in Appendix \ref{apd:OCO}.

\textbf{Applications to Boosting:}
There is a deep and well-known connection between Hedge and boosting
\cite{FreundSc97, SchapireFr12}.
In principle, every Hedge algorithm 
can be converted into a boosting algorithm; 
for instance, this is how AdaBoost was derived from EXP.
In the same way, {\ANH} can be converted into a new boosting 
algorithm that we call {\ANB}.  
See Appendix \ref{apd:ANB} for details and further background on boosting.
The main idea is to treat each training example as an ``action'',
and to rely on the Hedge algorithm to compute distributions over these 
examples which are used to train the weak hypotheses.
Typically, it is assumed that each of these has ``edge'' $\gamma$, 
meaning its accuracy on the training distribution is at least $1/2 + \gamma$.
The final hypothesis is a simple majority vote of the weak hypotheses.
To understand the prediction accuracy of a boosting algorithm, we often
study the training error rate and also the distribution
of margins, a well-established measure of confidence 
(see Appendix \ref{apd:ANB} for formal definitions).
Thanks to the adaptivity of {\ANH}, we can derive bounds on both the training
error and the distribution of margins after any number of rounds:


\begin{theorem}\label{thm:ANB}
After $T$ rounds,
the training error of {\ANB} is of order $\tilde{O}(\exp(-\frac{1}{3}T\gamma^2))$,
and the fraction of training examples with margin at most $\theta (\leq 2\gamma)$ is
of order $\tilde{O}(\exp(-\frac{1}{3}T(\theta-2\gamma)^2))$.
\end{theorem}


Thus, the training error decreases at roughly the same rate as AdaBoost.
In addition, this theorem implies that the fraction of examples with margin 
smaller than $2\gamma$ eventually goes to zero as $T$ gets large,
which means {\ANB} converges to the optimal margin $2\gamma$; 
this is known not to be true for AdaBoost (see \cite{SchapireFr12}).
Also, like AdaBoost, {\ANB} is an
adaptive boosting algorithm that does not require $\gamma$ or $T$ as a parameter. 
However, unlike AdaBoost, {\ANB} has the striking 
property that it completely ignores many examples on each round 
(by assigning zero weight), 
which is very helpful for the weak learning algorithm 
in terms of computational efficiency.
To test this, we conducted experiments to compare the efficiency of 
AdaBoost, ``NH-Boost'' 
(an analogous boosting algorithm derived from NormalHedge) and
{\ANB}. 
All details are in Appendix \ref{apd:experiments}.
Here we only briefly summarize the results. 
While the three algorithms have similar performance in terms of training and test error, 
{\ANB} is always the fastest one in terms of running time for the same number of rounds. 
Moreover, the average faction of examples with zero weight is
significantly higher for {\ANB} than for NH-Boost (see Table \ref{tab:results}).
On one hand, this explains why {\ANB} is faster 
(besides the reason that it does not require a numerical step).
On the other hand, this also implies that {\ANB} tends to achieve
larger margins, since zero weight is assigned to examples with large margin.
This is also confirmed by our experiments.

{\textbf{Acknowledgements.} 
Support for this research was provided by NSF Grant \#1016029.
The authors thank Yoav Freund for helpful discussions and 
the anonymous reviewers for their comments.}

\newpage
\bibliographystyle{plain}
{\small\bibliography{../references/ref.bib}}


\newpage
\appendix

\section{Summary of Drifting Game Variants}
\label{apd:DGv}
We study three different variants of drifting games throughout the paper,
which corresponds to the Hedge setting, the multi-armed bandit problem
and online convex optimization respectively.
The protocols of these variants are summarized below.

\begin{framed}
\centerline{\textbf{DGv1}}
Given: a loss function $L(s) = \1\{s \leq -R\}$. \\
For $t = 1, \ldots, T$: 
\begin{enumerate}
\item The player chooses a distribution $\p_t$ over $N$ chips.
\item The adversary decides the movement of each chip $z_{t,i} \in [-1,1]$ 
subject to $\p_t \cdot \z_t \geq 0$ and $|z_{t,i}-z_{t,j}|\leq 1$ for all $i$ and $j$.
\end{enumerate}
The player suffers loss $\sum_{i=1}^N L(\sum_{t=1}^T z_{t,i})$.
\end{framed}

\begin{framed}
\centerline{\textbf{DGv2}}
Given: a loss function $L(s) = \1\{s \leq -R\}$. \\
For $t = 1, \ldots, T$: 
\begin{enumerate}
\item The player chooses a distribution $\p_t$ over $N$ chips.
\item The adversary randomly decides the movement of each chip $z_{t,i} \geq -1$
subject to $\E_t[z_{t,i}] \leq 1, \E_t[z_{t,i}^2] \leq 1/p_{t,i}$ 
and $\E_t[\p_t \cdot \z_t] \geq 0$.
\end{enumerate}
The player suffers loss $\sum_{i=1}^N L(\sum_{t=1}^T z_{t,i})$.
\end{framed}

\begin{framed}
\centerline{\textbf{DGv3}}
Given: a compact convex set $S$, a loss function $L(s) = \1\{s \leq -R\}$. \\
For $t = 1, \ldots, T$: 
\begin{enumerate}
\item The player chooses a density function $p_t(\x)$ on $S$.
\item The adversary decides a function $z_t(\x): S \rightarrow [-1,1]$
subject to $\E_{\x \sim p_t} [z_t(\x)] \geq 0$.
\end{enumerate}
The player suffers loss $\int_{\x\in S} L(\sum_{t=1}^T z_{t}(\x)) d\x$.
\end{framed}

\section{Proof of Theorem \ref{thm:equivalence}}
\label{apd:equiv}

\begin{proof}
We first show that both conversions are valid. In Algorithm \ref{alg:hedge2drift},
it is clear that $\ell_{t,i} \geq 0$. Also, $\ell_{t,i} \leq 1$ is guaranteed
due to the extra restriction of DGv1. 
For Algorithm \ref{alg:drift2hedge}, $z_{t,i}$ lies in $B=[-1,1]$ since 
$\ell_{t,i} \in [0,1]$, and direct computation shows 
$\p_t \cdot \z_t = 0 \geq \beta (=0)$
and $|z_{t,i} - z_{t,j}| = |\ell_{t,i} - \ell_{t,j}| \leq 1$ for all $i$ and $j$. 

(1) For any choices of $\z_{t}$, we have
$$ \sum_{i=1}^N L(s_{T,i}) = \sum_{i=1}^N L\(\sum_{t=1}^N z_{t,i}\) 
\leq \sum_{i=1}^N L\(\sum_{t=1}^N \(z_{t,i} - \p_t\cdot\z_t\)\), $$
where the inequality holds since $\p_t\cdot\z_t$ is required to be nonnegative
and $L$ is a nonincreasing function.
By Algorithm \ref{alg:hedge2drift}, $z_{t,i} - \p_t\cdot\z_t$ is equal to
$\ell_{t,i} - \p_t\cdot\l_t$, leading to
$$ \sum_{i=1}^N L(s_{T,i}) \leq 
\sum_{i=1}^N L\(\sum_{t=1}^N \(\ell_{t,i} - \p_t\cdot\l_t\)\)
= \sum_{i=1}^N \1\left\{R \leq \sum_{t=1}^N \(\p_t\cdot\l_t - \ell_{t,i}\)\right\}.
$$
Since $\Reg_T^{(i+1)/N}(\H) < R \leq \Reg_T^{i/N}(\H)$,
we must have $\sum_{t=1}^N \(\p_t\cdot\l_t - \ell_{t,j}\) < R$
except for the best $i$ actions, which means $\sum_{i=1}^N L(s_{T,i}) \leq i$.
This holds for any choices of $\z_t$, so $L_T(\D_R) \leq i/N$.

(2) By Algorithm \ref{alg:drift2hedge} and the condition $L_T(D_R) < \epsilon$ , 
we have
$$ \frac{1}{N}\sum_{i=1}^N \1\left\{R \leq 
\sum_{t=1}^N \(\p_t\cdot\l_t - \ell_{t,i}\)\right\}
= \frac{1}{N}\sum_{i=1}^N L(s_{T,i}) \leq L_T(D_R) < \epsilon,
 $$
which means there are at most $\lceil N\epsilon \rceil - 1$ actions satisfying
$R \leq \sum_{t=1}^N \(\p_t\cdot\l_t - \ell_{t,i}\)$,
and thus $\sum_{t=1}^N \(\p_t\cdot\l_t - \ell_{t,i_\epsilon}\) < R$. 
Since this holds for any choices of $\l_t$, we have $\Reg_T^\epsilon(\H) < R$.
\end{proof}

\section{Summary of Hedge Algorithms and Proofs of Lemma \ref{lem:2-norm}, Lemma \ref{lem:anh} and 
Corollary \ref{cor:anh}}
\label{apd:lemmas}

\begin{table}[h]
\caption{Different algorithms derived from Algorithm \ref{alg:hedge_recipe},
and comparisons with NormalHedge}
\label{tab:algs}
\begin{center}
\begin{tabular}{|c|c|c|c||c|}
\hline
& EXP & 2-norm & \ANH & NormalHedge \\
\hline
$\Phi_T(s)$ & $e^{-\eta(s+R)}$ & $a[s]_-^2$ & 
$a\(e^{[s]_-^2/3T} - 1\) $ & N/A \\
\hline
$p_{t,i} \propto$ & $e^{-\eta s_{t-1,i}}$ & 
\pbox{10em}{$[s_{t-1,i}-1]_-^2 $ \\ $-[s_{t-1,i}+1]_-^2$} &
\pbox{10em}{$e^{[s_{t-1,i}-1]_-^2/3t}$ \\ $- e^{[s_{t-1,i}+1]_-^2/3t}$} &
\pbox{20em}{$-[s_{t-1,i}]_-e^{[s_{t-1,i}]_-^2/c}$  ($c$ is\\
 s.t. $\sum_{i} e^{[s_{t-1,i}]_-^2/c} = Ne) $} \\
\hline
$\Reg_T^\epsilon(\H)$ & $O\(\sqrt{T\ln \frac{1}{\epsilon}}\)$ & 
$O\(\sqrt{T/\epsilon}\)$ &
$O\(\sqrt{T\ln\frac{\ln T}{\epsilon}}\)$ & 
$O\(\sqrt{T\ln\frac{1}{\epsilon}}+\ln^2 N\)$ \\
\hline
Adaptive? & No & Yes & Yes & Yes \\
\hline
\end{tabular}
\end{center}
\end{table}

\begin{proof}[Proof of Lemma \ref{lem:2-norm}]
It suffices to show $[s-1]_-^2 + [s+1]_-^2 \leq 2[s]_-^2 + 2$.
When $s \geq 0$, $\LHS = [s-1]_-^2 \leq 1 < 2 = \RHS$.
When $s < 0$, $\LHS \leq (s-1)^2+(s+1)^2 = 2s^2 + 2 = \RHS$.
\end{proof}

\begin{proof}[Proof of Lemma \ref{lem:anh}]
Let $F(s) = \exp\(\frac{[s-1]_-^2}{dt}\) + \exp\(\frac{[s+1]_-^2}{dt}\)
- 2\exp\(\frac{[s]_-^2}{d(t-1)}\)$.
It suffices to show 
$$ F(s) \leq 2(b_{t} - b_{t-1}) = \exp\(\frac{4}{dt}\) - 1 ,$$
which is clearly true for the following 3 cases:
$$ F(s) = \begin{cases}
0 &\quad\text{if $s > 1$;} \\
\exp\(\frac{(s-1)^2}{dt}\) -1 < \exp\(\frac{1}{dt}\) - 1 
&\quad\text{if $0 < s \leq 1$;} \\
\exp\(\frac{(s-1)^2}{dt}\) + 1 - 2\exp\(\frac{s^2}{d(t-1)}\)
<  \exp\(\frac{4}{dt}\)  - 1 &\quad\text{if $-1 < s \leq 0$.}
\end{cases}$$
For the last case $s \leq -1$, if we can show that $F(s)$ is increasing
in this region, then the lemma follows.
Below, we show this by proving $F'(s)$ is nonnegative when $s \leq -1$. 

Let $h(s, c) = \frac{\partial\exp\(s^2/c\)}{\partial s} = 
\frac{2s}{c}\exp\(\frac{s^2}{c}\)$.
$F'(s)$ can now be written as
$$ F'(s) = h(s-1, c) + h(s+1,c) - 2h(s, c) + 2(h(s,c) - h(s, c')) ,$$
where $c = dt$ and $c' = d(t-1)$. 
Next we apply (one-dimensional) Taylor expansion to $h(s-1,c)$ and $h(s+1,c)$ 
around $s$, and $h(s,c')$ around $c$, leading to
\begin{align*}
F'(s) &= \sum_{k=1}^\infty \frac{(-1)^k}{k!} \frac{\partial^k h(s,c)}{\partial s^k}
+ \sum_{k=1}^\infty \frac{1}{k!} \frac{\partial^k h(s,c)}{\partial s^k}
- 2\sum_{k=1}^\infty \frac{(c'-c)^k}{k!} \frac{\partial^k h(s,c)}{\partial c^k} \\
&= 2 \sum_{k=1}^\infty \(\frac{1}{(2k)!} \frac{\partial^{2k} h(s,c)}{\partial s^{2k}} 
- \frac{(-d)^k}{k!} \frac{\partial^k h(s,c)}{\partial c^k} \).
\end{align*}
Direct computation (see Lemma \ref{lem:partialD} below) shows that 
$\frac{\partial^k h(s,c)}{\partial c^k}$
and $\frac{\partial^{2k} h(s,c)}{\partial s^{2k}}$ share exact same forms 
only with different constants:
\begin{equation}\label{equ:partialD}
\begin{split}
 \frac{\partial^{k} h(s,c)}{\partial c^{k}} 
&= \exp\(\frac{s^2}{c}\) \sum_{j=0}^k (-1)^k\alpha_{k,j} \cdot \frac{s^{2j+1}}{c^{k+j+1}} ,
 \\
\frac{\partial^{2k} h(s,c)}{\partial s^{2k}}
&= \exp\(\frac{s^2}{c}\) \sum_{j=0}^k \beta_{k,j} \cdot \frac{s^{2j+1}}{c^{k+j+1}},
\end{split}
\end{equation}
where $\alpha_{k,j}$ and $\beta_{k,j}$ are recursively defined as:
\begin{equation}\label{equ:alpha_beta}
\begin{split}
\alpha_{k+1,j} &= \alpha_{k,j-1} + (k+j+1)\alpha_{k,j}, \\
\beta_{k+1,j} &= 4\beta_{k,j-1} + (8j+6)\beta_{k,j} + (2j+3)(2j+2)\beta_{k,j+1},
\end{split}
\end{equation}
with initial values $\alpha_{0,0} = \beta_{0,0} = 2$
(when $j \not\in \{0,\ldots,k\}$, $\alpha_{k,j}$ and $\beta_{k,j}$ are all defined to be $0$).
Therefore, $F'(s)$ can be further simplified as
$$ F'(s) =  2\exp\(\frac{s^2}{c}\)  \sum_{k=1}^\infty \sum_{j=0}^k \frac{s^{2j+1}}{c^{k+j+1}}
\(\frac{\beta_{k,j}}{(2k)!} - \frac{d^k\alpha_{k,j}}{k!} \).
$$
Since $s$ is negative, it suffices to show that 
$\frac{\beta_{k,j}}{(2k)!} \leq \frac{d^k\alpha_{k,j}}{k!}$ holds for all $k$ and $j$,
which turns out to be true as long as $d \geq 3$,  
as shown by induction in the technical lemma \ref{lem:alpha_beta} below. 
To sum up, $\Phi_{t}(s-1)+\Phi_{t}(s+1) \leq 2\Phi_{t-1}(s)$
for all $s \in \R$ and $t = 2, \ldots, T$. 

Finally, we need to show that Eq. (2) still holds. 
The inequality we just proved above implies
$ \sum_{i} \Phi_{t}(s_{t,i}) \leq \sum_i \Phi_{t-1}(s_{t-1,i})$
for $t = 2, \ldots, T$, as shown in Theorem \ref{thm:hedge_recipe}.
Thus the only thing we need to show here is the case when $t=1$.
Note that $\Phi_{1}(s-1)+\Phi_{1}(s+1) \leq 2\Phi_0(s)$ does not
hold for all $s$ apparently. 
However, in order to prove $\sum_{i} \Phi_{1}(s_{1,i}) \leq \sum_i \Phi_0(s_{0,i})$, 
we in fact only need a much weaker statement: $\Phi_{1}(-1)+\Phi_{1}(1) \leq 2\Phi_0(0)$
since $s_{0,i} \equiv 0$.
This is equivalent to 
$ \exp\(1/d\) - 1 \leq \exp\(4/d\) - 1$,
which is true trivially. 
\end{proof}

\begin{lemma}\label{lem:partialD}
Let $h(s,c) = \frac{2s}{c}\exp\(\frac{s^2}{c}\)$. 
The partial derivatives of $h(s,c)$ satisfy Eq. \eqref{equ:partialD} 
and \eqref{equ:alpha_beta}.
\end{lemma}
\begin{proof}
The base case holds trivially. 
Assume Eq. \eqref{equ:partialD} holds for a fixed $k$.
Then we have
\begin{align*}
 \frac{\partial^{k+1} h(s,c)}{\partial c^{k+1}} 
&= \exp\(\frac{s^2}{c}\) \sum_{j=0}^k (-1)^k\alpha_{k,j} \cdot 
\(-\frac{s^2}{c^2} \frac{s^{2j+1}}{c^{k+j+1}} -(k+j+1)\frac{s^{2j+1}}{c^{k+j+2} }\) \\
&= \exp\(\frac{s^2}{c}\) \sum_{j=0}^k (-1)^{k+1} \alpha_{k,j} \cdot 
\( \frac{s^{2(j+1)+1}}{c^{(k+1)+(j+1)+1}} + (k+j+1)\frac{s^{2j+1}}{c^{(k+1)+j+1} }\) \\
&= \exp\(\frac{s^2}{c}\) \sum_{j=0}^{k+1} (-1)^{k+1} 
\(\alpha_{k,j-1} + (k+j+1)\alpha_{k,j} \)  \cdot \frac{s^{2j+1}}{c^{(k+1)+j+1} } \\
&= \exp\(\frac{s^2}{c}\) \sum_{j=0}^{k+1} (-1)^{k+1} 
\alpha_{k+1,j}   \cdot \frac{s^{2j+1}}{c^{(k+1)+j+1} } ,
\end{align*}
and 
\begin{align*}
\frac{\partial^{2(k+1)} h(s,c)}{\partial s^{2(k+1)}}
&= \left. \partial \left[ \exp\(\frac{s^2}{c}\) \sum_{j=0}^k \beta_{k,j} \cdot 
\(\frac{2s^{2j+2}}{c^{k+j+2}} + (2j+1)\frac{s^{2j}}{c^{k+j+1}}\)
\right] \middle/ \partial s \right. \\
&= \exp\(\frac{s^2}{c}\) \sum_{j=0}^k \beta_{k,j} \cdot 
\(\frac{4s^{2j+3}}{c^{k+j+3}} +
(8j+6)\frac{s^{2j+1}}{c^{k+j+2}} + 
 (2j+1)2j \frac{s^{2j-1}}{c^{k+j+1}} \) \\
 &= \exp\(\frac{s^2}{c}\) \sum_{j=0}^{k+1} \(
 4\beta_{k,j-1} + (8j+6)\beta_{k, j} + (2j+3)(2j+2)\beta_{k, j+1}\)
 \cdot \frac{s^{2j+1}}{c^{k+j+2}}  \\
&= \exp\(\frac{s^2}{c}\) \sum_{j=0}^{k+1} \beta_{k+1,j}  \cdot \frac{s^{2j+1}}{c^{k+j+2}}, 
\end{align*}
concluding the proof.
\end{proof}

\begin{lemma}\label{lem:alpha_beta}
Let $\alpha_{k,j}$ and $\beta_{k,j}$ be defined as in Eq. \eqref{equ:alpha_beta}.
Then $\frac{\beta_{k,j}}{(2k)!} \leq \frac{d^k\alpha_{k,j}}{k!}$ holds for all 
$k \geq 0$ and $j \in \{0, \ldots, k\}$ when $d \geq 3$.
\end{lemma}
\begin{proof}
We prove the lemma by induction on $k$. 
The base case $k = 0$ is trivial. 
Assume $\frac{\beta_{k,j}}{(2k)!} \leq \frac{d^k\alpha_{k,j}}{k!}$ holds
for a fixed $k$ and all $j \in \{0, \ldots, k\}$, then we have $\forall j$,
\begin{align*}
\frac{\beta_{k+1, j}}{(2k+2)!}  &=
\frac{4\beta_{k,j-1} + (8j+6)\beta_{k, j} + (2j+3)(2j+2)\beta_{k, j+1}}{(2k+2)!} \\
&\leq \frac{d^k\(4\alpha_{k,j-1} + (8j+6)\alpha_{k, j} + 
(2j+3)(2j+2)\alpha_{k, j+1}\)}{(2k+2)(2k+1)k!}.
\end{align*}
We need to show that the above expression is at most 
$d^{k+1}\alpha_{k+1,j}/(k+1)!$,
which, after arrangements, is equivalent to
$2\alpha_{k,j-1} + (4j+3)\alpha_{k, j} + 
(2j+3)(j+1)\alpha_{k, j+1} \leq d(2k+1)\alpha_{k+1,j}.$
We will prove this by another induction on $k$.
Then the lemma follows.

The base case ($k=0$) is simplified to $6 \leq 2d$,
which is true by our assumption $d \geq 3$.
Assume the inequality holds for a fixed $k$, 
then by the definition of $\alpha_{k,j}$, one has
\begin{align*}
&2\alpha_{k+1,j-1} + (4j+3)\alpha_{k+1, j} + 
(2j+3)(j+1)\alpha_{k+1, j+1}  \\
=\;& \(2\alpha_{k,j-2} + (4j+3)\alpha_{k, j-1} + (2j+3)(j+1)\alpha_{k, j} \) + \\
\quad &\(2(k+j)\alpha_{k,j-1} + (4j+3)(k+j+1)\alpha_{k,j} + 
(2j+3)(j+1)(k+j+2)\alpha_{k, j+1} \) \\
=\;& \(2\alpha_{k,j-2} + (4j-1)\alpha_{k, j-1} + (2j+1)j\alpha_{k, j} \) + \\
\quad & (k+j+2)\(2\alpha_{k,j-1} + (4j+3)\alpha_{k,j} + 
(2j+3)(j+1)\alpha_{k, j+1} \) \\
\leq\;& d(2k+1) (\alpha_{k+1, j-1} + (k+j+2)\alpha_{k+1, j}) \tag{by induction}\\
=\;& d(2k+1)\alpha_{k+2,j} \\
\leq\;& d(2k+3)\alpha_{k+2,j},
\end{align*}
completing the induction. 
\end{proof}

\begin{proof}[Proof of Corollary \ref{cor:anh}]
Recall that $\Phi_T(s) \geq \1\left\{ s \leq -\sqrt{dT \ln \(\tfrac{1}{a} + 1\)} \right\}$.
So by setting $\Phi_0(0) = a(1-b_0) < \epsilon$ 
and applying Theorem \ref{thm:hedge_recipe}, we arrive at
$$\Reg_T^\epsilon(\H) \leq \sqrt{dT \ln \(\frac{1-b_0}{\epsilon} + 1\)}. $$
It suffices to upper bound $1-b_0$, which, by definition, is
$ \frac{1}{2}\sum_{t = 1}^T \(\exp\(\frac{4}{dt}\) - 1\)$.
Since $e^x -1 \leq \frac{e^c-1}{c} x$ for any $x \in [0, c]$,
we have
$$ \sum_{t = 1}^T \(\exp\(\frac{4}{dt}\) - 1\) 
\leq (e^{4/d}-1)\sum_{t = 1}^T \frac{1}{t}
\leq (e^{4/d}-1)(\ln T + 1).
$$
Plugging $d=3$ gives the corollary.
\end{proof}

\section{A General MAB Algorithm and Regret Bounds}
\label{apd:bandit}
\SetAlCapSkip{.2em}
\IncMargin{.5em}
\begin{algorithm}[H]
\caption{A General MAB Algorithm}
\label{alg:bandit_recipe}

\SetKwInOut{Input}{Input}
\SetKw{DownTo}{down to}
\Input{A convex, nonincreasing, nonnegative function $\Phi_T(s) \in \C^2$,
with nonincreasing second derivative.
}

\For{ $t = T$ \DownTo $1$} {
Find a convex function $\Phi_{t-1}(s)$ s.t.  
the conditions of Theorem \ref{thm:DGv2} hold. 
}
Set: $\s_0 = \0$. \\
\For{ $t = 1$ \KwTo $T$} {
Set:  $p_{t,i} \propto \Phi_t(s_{t-1,i}-1) - \Phi_t(s_{t-1,i}+1)$. \\
Draw $i_t \sim \p_t$ and receive loss $\ell_{t, i_t}$. \\
Set:  $z_{t,i} =  \1\{i = i_t\} \cdot\ell_{t, i_t}/p_{t,i_t} -\ell_{t, i_t}, \;\forall i$. \\
Set:  $\s_t = \s_{t-1} + \z_t$. \\
}
\end{algorithm}
\DecMargin{.5em}

\begin{theorem}\label{thm:DGv2}
Suppose $\Phi_t(s)$ is convex, twice continuously differentiable (i.e. $\Phi_t(s) \in \C^2$),
have nonincreasing second derivative, and satisfies:
\begin{equation}\label{equ:bandit_potential}
\(\tfrac{1}{2}+N\alpha_t\)\Phi_{t}(s-1)+\(\tfrac{1}{2}-N\alpha_t\)\Phi_{t}(s+1)
\leq \Phi_{t-1}(s), \forall s\in\R
\end{equation} 
where $\alpha_t = \frac{1}{2}\max_s \frac{\Phi_t''(s-1)}{\Phi_t(s-1)-\Phi_t(s+1)}$.
If the player's strategy is such that $p_{t,i} \propto 
\Phi_t(s_{t-1,i}-1) - \Phi_t(s_{t-1,i}+1)$,
then Eq. \eqref{equ:monotonic_potentials} holds in expectation.
\end{theorem}

\begin{proof}[Proof of Theorem \ref{thm:DGv2}]
As discussed before, the main difficulty here is the unboundedness of $z_{t,i}$.
However, the expectation of $z_{t,i}$ is still in $[-1,1]$ as in DGv1.
To exploit this fact, we apply Taylor's theorem to $\Phi_t(s_{t-1,i}+z_{t,i})$
to the second order term:
\begin{align*}
\Phi_t(s_{t,i}) &= \Phi_t(s_{t-1,i}+z_{t,i}) \\
&= \Phi_t(s_{t-1,i}) + \Phi_t'(s_{t-1,i})z_{t,i} + \tfrac{1}{2}\Phi_t''(\xi_{t,i})z_{t,i}^2 \\
&\leq \Phi_t(s_{t-1,i}) + \Phi_t'(s_{t-1,i})z_{t,i} + \tfrac{1}{2}\Phi_t''(s_{t-1,i}-1)z_{t,i}^2,
\end{align*}
where $\xi_{t,i}$ is between $s_{t-1,i}+z_{t,i}$ and $s_{t-1,i}$, and
the inequality holds because $\Phi_t''(s)$ is nonincreasing and 
$z_{t,i} \geq -1$ by assumption.
Now taking expectation on both sides with respect to the randomness of $z_{t,i}$,
using the convexity of $\Phi_t(s)$, and plugging the assumption 
$\E_t[z_{t,i}^2] \leq 1/p_{t,i}$ give:
\begin{align*}
\E_t[\Phi_t(s_{t,i})] &\leq 
\Phi_t(s_{t-1,i}) + \Phi_t'(s_{t-1,i})\E_t[z_{t,i}] + 
\tfrac{1}{2}\Phi_t''(s_{t-1,i}-1)\E_t[z_{t,i}^2] \\
&\leq \Phi_t\(s_{t-1,i} + \E_t[z_{t,i}]\) + \tfrac{1}{2} \Phi_t''(s_{t-1,i}-1)/p_{t,i}.
\end{align*}
Let $w_{t,i} = \frac{1}{2}\(\Phi_t(s_{t-1,i}-1) - \Phi_t(s_{t-1,i}+1)\)$.
Further plugging $\p_{t,i} \propto w_{t,i}$ and summing over all $i$, we arrive at
\begin{align*}
\sum_{i=1}^N \E_t[\Phi_t(s_{t,i})] &\leq
\sum_{i=1}^N \( \Phi_t\(s_{t-1,i} + \E_t[z_{t,i}]\) + \frac{\Phi_t''(s_{t-1,i}-1)}{2w_{t,i}}
\cdot \sum_{i=1}^N w_{t,i} \) \\
&\leq \sum_{i=1}^N \( \Phi_t\(s_{t-1,i} + \E_t[z_{t,i}]\) + 2\alpha_t \sum_{i=1}^N w_{t,i} \) 
\tag{\text{by the defintion of $\alpha_t$}} \\
&= \sum_{i=1}^N \( \Phi_t\(s_{t-1,i} + \E_t[z_{t,i}]\) + 2N\alpha_t w_{t,i} \) .
\end{align*}
Since $\E_t[\p_t \cdot \z_t] \geq 0$ implies $\sum_{i=1}^N w_{t,i} \E_t[z_{t,i}] \geq 0$,
we thus have
\begin{align*}
\sum_{i=1}^N \E_t[\Phi_t(s_{t,i})] &\leq
\sum_{i=1}^N \( \Phi_t\(s_{t-1,i} + \E_t[z_{t,i}]\) + w_{t,i}\E_t[z_{t,i}] + 2N\alpha_t w_{t,i} \) \\
&\leq \sum_{i=1}^N \( \max_{z\in [-1,+1]} \(\Phi_t\(s_{t-1,i} + z\) + w_{t,i} z \)
+ 2N\alpha_t w_{t,i} \) \\
&= \sum_{i=1}^N \( \max_{z\in \{-1,+1\}} \(\Phi_t\(s_{t-1,i} + z\) + w_{t,i} z \)
+ 2N\alpha_t w_{t,i} \) \tag{\text{by the convexity of $\Phi_t(s)$}}\\
&= \sum_{i=1}^N \(  \(\tfrac{1}{2}+N\alpha_t\)\Phi_{t}(s_{t-1,i}-1)+
\(\tfrac{1}{2}-N\alpha_t\)\Phi_{t}(s_{t-1,i}+1)\)  \\
&\leq \sum_{i=1}^N \Phi_{t-1}(s_{t-1,i}). \tag{by assumption}
\end{align*}
The theorem follows by taking expectation on both sides with respect to the past 
(i.e. the randomness of $\z_{1}, \ldots, \z_{t-1}$).
\end{proof}

\begin{theorem}\label{thm:bandit_recipe}
For Algorithm \ref{alg:bandit_recipe},
if $R$ and $\epsilon$ are such that $\Phi_0(0) < \epsilon$
and $\Phi_T(s) \geq \1\{s \leq -R\}$ for all $s\in\R$, 
then $\E[\sum_{t=1}^T \ell_{t, i_t} - 
\sum_{t=1}^T \ell_{t, i_\epsilon}] < R$ for any non-oblivious adversary.
Moreover, using $\Phi_T(s) = \exp(-\eta(s+R))$ 
(and let Eq. \eqref{equ:bandit_potential} hold with equality) gives exactly
the EXP3 algorithm with regret $O(\sqrt{TN\ln(1/\epsilon)})$.
\end{theorem}

\begin{proof}[Proof of Theorem \ref{thm:bandit_recipe}]
We first show that Algorithm \ref{alg:bandit_recipe} converts the multi-armed 
bandit problem to a valid instance of DGv2. 
It suffices to prove that $z_{t,i} = \1\{i = i_t\} \cdot\ell_{t, i_t}/p_{t,i_t} -\ell_{t, i_t}$
satisfies all conditions defined in DGv2, as shown below ($z_{t,i} \geq -1$ is trivial):
$$ \E_t[z_{t,i}] = \ell_{t, i} - \p_t \cdot \l_t \leq 1 ,$$
$$ \E_t[z_{t,i}^2] 
= p_{t,i} \(\frac{\ell_{t,i}}{p_{t,i}} - \ell_{t,i}\)^2 + \sum_{j\neq i} p_{t,j} \ell_{t, j}^2 
\leq p_{t,i} \(\frac{1}{p_{t,i}} - 1\)^2 + \sum_{j\neq i} p_{t,j} 
= \frac{1}{p_{t,i}} - 1 \leq \frac{1}{p_{t,i}},
$$
$$ \E_t[\p_t \cdot \z_t] = \E_t\left[\ell_{t,i_t} - \sum_{j=1}^N p_{t,j}\ell_{t,i_t} \right]  = 0.$$
Therefore, we can apply Theorem \ref{thm:DGv2} directly, arriving at:
$$ \frac{1}{N}\sum_{i=1}^N \E[\Phi_T(s_{T,i})] \leq  \cdots 
 \leq \frac{1}{N}\sum_{i=1}^N \E[\Phi_{0}(s_{0,i})]  = \Phi_{0}(0) \leq \epsilon. $$
On the other hand, by applying Jensen' inequality, we have
$$\E[\Phi_T(s_{T,i})] \geq \Phi_T(\E[s_{T,i}]) \geq \1\{\E[s_{T,i}] \leq -R \}. $$
Note that $\E[s_{T,i}]$ is equal to 
$\E\left[\sum_{t=1}^T\(\ell_{t, i} - \ell_{t,i_t}\)\right]$. We thus know
$$ \frac{1}{N}\sum_{i=1}^N \1\left\{ \E\left[\sum_{t=1}^T\(\ell_{t, i} - 
\ell_{t,i_t}\)\right] \leq -R \right\}  < \epsilon ,$$
which implies $\E\left[\sum_{t=1}^T \ell_{t, i_t} - 
\sum_{t=1}^T \ell_{t, i_\epsilon}\right] < R$ for any non-oblivious adversary
for the exact same argument used in the proof of Theorem \ref{thm:hedge_recipe}.

Finally, we show how to recover EXP3 using Algorithm \ref{alg:bandit_recipe}
with input $\Phi_T(s) = \exp(-\eta(s+R))$.
To compute $\Phi_t(s)$ for $t < T$, 
we simply use Eq. \eqref{equ:bandit_potential} with equality. 
One can verify using induction that
$$ \Phi_t(s) = \exp\(-\eta(s+R)\)\(\frac{e^\eta+e^{-\eta}+Ne^\eta\eta^2}{2}\)^{T-t}, $$
$$ \alpha_t 
= \frac{1}{2}\max_s \frac{\eta^2\Phi_t(s-1)}{\Phi_t(s-1)-\Phi_t(s+1)}
= \frac{e^\eta \eta^2}{2(e^\eta - e^{-\eta})} ,$$
$$ \Phi_t'''(s) = -\eta^3 \Phi_t(s) \leq 0 .$$
The player's strategy is thus 
$\p_{t,i} \propto \exp(-\eta\sum_{\tau=1}^{t-1} \hat\ell_{\tau,i})$
(recall $\hat\ell_{t,i} = \1\{i = i_t\}\cdot\ell_{t,i_t}/p_{t, i_t}$ is 
the estimated loss),
which is exactly the same as EXP3 (in fact a simplified version
of the original EXP3, see for example \cite{Shalevshwartz11}).
Moreover, the regret can be computed by setting $\Phi_0(0) = \epsilon$,
leading to 
\begin{align*}
R &= \frac{1}{\eta}\ln\(\frac{1}{\epsilon}\) + 
\frac{T}{\eta}\ln \(\frac{e^\eta+e^{-\eta}}{2}+\frac{1}{2}Ne^\eta\eta^2\) \\
&\leq \frac{1}{\eta}\ln\(\frac{1}{\epsilon}\) + 
\frac{T}{\eta}\ln \(e^{\eta^2/2}+\frac{1}{2}Ne^\eta\eta^2\) 
\tag{\text{by Hoeffding's Lemma}}\\
&\leq \frac{1}{\eta}\ln\(\frac{1}{\epsilon}\) + 
\frac{T}{\eta} \(\frac{\eta^2}{2} + \frac{1}{2}Ne^{\eta-\frac{\eta^2}{2}}\eta^2\)
\tag{$\ln(1+x) \leq x$}
\end{align*}
If $\eta \leq 1$ so that $e^{\eta-\eta^2/2} \leq \sqrt{e}$, then we have
$ R \leq \frac{1}{\eta}\ln(\frac{1}{\epsilon}) + 
T\eta \(\frac{1}{2} + \frac{N\sqrt{e}}{2} \)$,
which is $\sqrt{2T(1+N\sqrt{e})\ln(1/\epsilon)}$
after optimally choosing $\eta$ ($\eta \leq 1$ will be satisfied when $T$ is large enough).
\end{proof}

\section{A General OCO Algorithm and Regret Bounds}
\label{apd:OCO}
\SetAlCapSkip{.2em}
\IncMargin{.5em}
\begin{algorithm}[H]
\caption{A General OCO Algorithm}
\label{alg:OCO_recipe}

\SetKwInOut{Input}{Input}
\SetKw{DownTo}{down to}
\Input{A convex, nonincreasing, nonnegative function $\Phi_T(s)$
}

\For{ $t = T$ \DownTo $1$} {
Find a convex function $\Phi_{t-1}(s)$ s.t.  $\forall s,$
$\Phi_{t}(s-1)+\Phi_{t}(s+1) \leq 2\Phi_{t-1}(s).$
}
Set: $s_0(x) \equiv 0$. \\
\For{ $t = 1$ \KwTo $T$} {
Predict $\x_t = \E_{\x \sim p_t} [\x]$ where $p_t$ is such that
$p_t(\x) \propto \Phi_t(s_{t-1}(\x)-1) - \Phi_t(s_{t-1}(\x)+1)$. \\
Receive loss function $f_t$ from the adversary. \\
Set:  $z_t(\x) = f_t(\x) - f_t(\x_t)$. \\
Set:  $s_t(\x) = s_{t-1}(\x) + z_t(\x)$. \\
}
\end{algorithm}
\DecMargin{.5em}

\textbf{Definition of $\epsilon$-regret in the OCO setting}:
Let $S_\epsilon \subset S$ be such that the ratio of its volume and
the one of $S$ is $\epsilon$ and also 
$ \sum_{t=1}^T f_t(\x') \leq \sum_{t=1}^T f_t(\x) $ for all
$\x' \in S_\epsilon$ and $\x \in S\backslash S_\epsilon$
(it is clear that such set exists).
Then $\epsilon$-regret is defined as
$\Reg_T^\epsilon(\x_{1:T}, f_{1:T}) = \sum_{t=1}^T  f_t(\x_t) 
- \inf_{\x\in S\backslash S_\epsilon} \sum_{t=1}^T f_t(\x)$.

\begin{theorem}\label{thm:OCO_recipe}
For Algorithm \ref{alg:OCO_recipe},
if $R$ is such that 
$\Phi_T(s) \geq \1\{s \leq -R\}$
and $\Phi_0(0) < \epsilon$, then we have
$\Reg_T^\epsilon(\x_{1:T}, f_{1:T}) < R$ and 
$\Reg_T(\x_{1:T}, f_{1:T}) < R + T\epsilon^{1/d}$.
Specifically, if $R = O(\sqrt{T\ln(1/\epsilon)})$,
then setting $\epsilon = T^{-d}$ gives
$\Reg_T(\x_{1:T}, f_{1:T}) = O(\sqrt{dT\ln T})$.
\end{theorem}

\begin{proof}[Proof of Theorem \ref{thm:OCO_recipe}]
Let $w_t(\x) = \frac{1}{2}\(\Phi_t(s_{t-1}(\x)-1) - \Phi_t(s_{t-1}(\x)+1)\)$.
Similarly to the Hedge setting,  the ``sum'' of potentials never increases:
$$
\int_{\x \in S} \Phi_{t}(s_{t}(\x)) d\x
\leq \int_{\x \in S} \(\Phi_{t}(s_{t-1}(\x) + z_t(\x)) + w_t(\x)z_t(\x)\) d\x  
\leq \int_{\x \in S} \Phi_{t-1}(s_{t-1}(\x)) d\x .$$
Here, the first inequality is due to $\E_{\x \sim p_t} [z_t(\x)] \geq 0$,
and the second inequality holds for the exact same reason as in the case
for Hedge.
Therefore, we have
$$ \int_{\x \in S} \1\{s_T(\x) \leq -R\} d\x  \leq 
\int_{\x \in S} \Phi_{T}(s_{T}(\x)) d\x \leq \cdots \leq
\int_{\x \in S} \Phi_{0}(0) d\x < \epsilon V, $$
where $V$ is the volume of $S$.
Recall the construction of $S_\epsilon$.
There must exist a point $\x' \in S_\epsilon$ such that
$s_T(\x') > -R$, otherwise 
$\int_{\x} \1\{s_T(\x) \leq -R\} d\x$
would be at least $\epsilon V$.
Unfolding $s_T(\x')$, we arrive at
$ \sum_t f_t(\x_t) - \sum_t f_t(\x')  < R $.
Using the fact $\sum_t f_t(\x') \leq \inf_{\x\in S\backslash S_\epsilon} \sum_t f_t(\x)$
gives the bound for $\epsilon$-regret.

Next consider a shrunk version of $S$:
$S_\epsilon' = \{(1-\epsilon^{\frac{1}{d}})\x^* + \epsilon^{\frac{1}{d}}\x : \x\in S\} $
where $\x^* \in\arg\min_\x \sum_t f_t(\x)$.
Then $\int_{\x \in S} \1\{s_T(\x) \leq -R\} d\x$ is at least
$$ 
 \int_{\x \in S_\epsilon'} \1\{s_T(\x) \leq -R\} d\x
= \epsilon \int_{\x \in S} \1\left\{s_T\((1-\epsilon^{\frac{1}{d}})\x^* + 
\epsilon^{\frac{1}{d}}\x\) \leq -R \right\} d\x,
$$
which, by the convexity and the boundedness of $f_t(\x)$, is at least
\begin{align*}
&\epsilon \int_{\x \in S} \1\left\{\sum_{t=1}^T\( (1-\epsilon^{\frac{1}{d}}) f_t(\x^*) + 
\epsilon^{\frac{1}{d}} f_t(\x) - f_t(\x_t) \)\leq -R \right\} d\x \\
\geq&\;  \epsilon \int_{\x \in S} \1\left\{\sum_{t=1}^T\( f_t(\x^*) - f_t(\x_t)\)
\leq -R-T\epsilon^{\frac{1}{d}}  \right\} d\x \\
=&\;  \epsilon V \cdot
\1\left\{\sum_{t=1}^T\( f_t(\x^*) - f_t(\x_t)\) \leq -R-T\epsilon^{\frac{1}{d}}  \right\} .
\end{align*}
Following the previous discussion, the expression in the last line above
is strictly less than $\epsilon V \cdot$, 
which means that the value of the indicator function has to be 0,
namely, $\Reg_T(\x_{1:T}, f_{1:T}) < R + T\epsilon^{1/d}$.
\end{proof}

\section{{\ANB}, NH-Boost and Proof of Theorem \ref{thm:ANB}}
\label{apd:ANB}

\SetAlCapSkip{.5em}
\IncMargin{.5em}
\begin{algorithm}[H]
\caption{{\ANB}}
\label{alg:ANB}

\SetKwInOut{Input}{Input}
\SetKwInOut{Output}{Output}
\Input{Training examples $(\x_i, y_i) \in \R^d \times \{-1,+1\},  i=1,\ldots,N.$}
\Input{A weak learning algorithm.}
\Input{Number of rounds $T$.}
\Output{A Hypothesis $H(\x) : \R^d \rightarrow \{-1,+1\}$.}

Set: $\s_0 = \0$. \\
\For{ $t = 1$ \KwTo $T$} {
Set:  $p_{t,i} \propto \exp\([s_{t-1,i}-1]_-^2/3t\) - \exp\([s_{t-1,i}+1]_-^2/3t\), \;\forall i$. \\
Invoke the weak learning algorithm to get $h_t$ with edge 
$ \gamma_t = \frac{1}{2}\sum_i p_{t,i}y_i h_t(\x_i) $. \\
Set:  $s_{t,i} = s_{t-1, i} + \frac{1}{2} y_i h_t(\x_i) - \gamma_t, \;\forall i $. \\
}
Set:  $H(\x) = \sign(\sum_{t=1}^T h_t(\x))$. 
\end{algorithm}
\DecMargin{.5em}

\SetAlCapSkip{.5em}
\IncMargin{.5em}
\begin{algorithm}[H]
\caption{NH-Boost}
\label{alg:NB}

\SetKwInOut{Input}{Input}
\SetKwInOut{Output}{Output}
\Input{Training examples $(\x_i, y_i) \in \R^d \times \{-1,+1\},  i=1,\ldots,N.$}
\Input{A weak learning algorithm.}
\Input{Number of rounds $T$.}
\Output{A Hypothesis $H(\x) : \R^d \rightarrow \{-1,+1\}$.}

Set: $\s_0 = \0$. \\
\For{ $t = 1$ \KwTo $T$} {
\eIf{$t = 1$} {
 Set: $\p_1$ to be a uniform distribution. \\ 
}
{Find: $c$ such that $\sum_{i=1}^N \exp\([s_{t-1,i}]_-^2/c\)= Ne$.  \\
 Set:  $p_{t,i} \propto -[s_{t-1,i}]_-\exp\([s_{t-1,i}]_-^2/c\), \;\forall i$. \\
 }
Invoke the weak learning algorithm to get $h_t$ with edge 
$ \gamma_t = \frac{1}{2}\sum_i p_{t,i}y_i h_t(\x_i) $. \\
Set:  $s_{t,i} = s_{t-1, i} + \frac{1}{2} y_i h_t(\x_i) - \gamma_t, \;\forall i $. \\
}
Set:  $H(\x) = \sign(\sum_{t=1}^T h_t(\x))$. 
\end{algorithm}
\DecMargin{.5em}

In the boosting setting for binary classification, 
we are given a set of training examples $(\x_i, y_i)_{i = 1, \ldots, N}$
where $\x_i \in \R^d$ is an example and  $y_i \in \{-1,+1\}$ is its label.
A boosting algorithm proceeds for $T$ rounds.
On each round, a distribution $\p_t$ over the examples is computed and fed into
a weak learning algorithm which returns a ``weak'' hypothesis 
$h_t: \R^d \rightarrow \{-1,+1\}$ with a guaranteed small edge, that is, 
$\gamma_t = \frac{1}{2}\sum_i p_{t,i}y_i h_t(\x_i) \geq \gamma > 0$. 
At the end, a linear combination of all $h_t$ is computed as the final ``strong'' hypothesis
which is expected to have low training error and potentially low generalization error.

The conversion of a Hedge algorithm into a boosting algorithm is to 
treat each example as an ``action'' and set $\ell_{t, i} = \1\{h_t(\x_i) = y_i\}$
so that the booster tends to increase weights for those ``hard'' examples.
The final hypothesis is a simple majority vote of all $h_t$, that is,
$H(\x) = \sign(\sum_t h_t(\x))$
where $\sign(x)$ is the sign function that outputs $1$ if $x$ is positive,
and $-1$ otherwise.
The {\it margin} of example $\x_i$ is defined as $\frac{1}{T}\sum_{t=1}^T y_i h_t(\x_i)$,
that is, the difference between the fractions of correct hypotheses
and incorrect hypotheses on this example.
The boosting algorithms derived from {\ANH} and NormalHedge 
in this way are given in Algorithm \ref{alg:ANB} and \ref{alg:NB}.


\begin{proof}[Proof the Theorem \ref{thm:ANB}]
Let $(\tx_i, \ty_i)_{i = 1, \ldots, N}$ be a permutation of the training examples such
that their margins are sorted from smallest to largest:
$ \sum_{t} \ty_1 h_t(\tx_1) \leq \cdots \leq \sum_{t} \ty_N h_t(\tx_N)$,
which also implies 
$ \sum_{t} \1\{h_t(\tx_1) = \ty_1\} \leq \cdots \leq \sum_{t} \1\{h_t(\tx_N) = \ty_N \}$.
Recall that {\ANH} is essentially playing a Hedge game
using {\ANH} with loss $\ell_{t, i} = \1\{h_t(\x_i) = y_i\}$.
Therefore, the $\epsilon$-regret bound for the Hedge setting 
together with the assumption on the weak learning algorithm implies:
$\forall j \in \{1,\ldots,N\},$
\begin{equation}\label{equ:hedge_and_boosting}
\frac{1}{2} + \gamma
\leq \frac{1}{T}\sum_{t=1}^T \sum_{i=1}^N \p_{t, i}\1\{h_t(\x_i) = y_i \} 
\leq \frac{1}{T}\sum_{t=1}^T \1\{h_t(\tx_j) = \ty_j \} + \frac{\Reg_T^{j/N}}{T},
\end{equation}
where $\Reg_T^{j/N} = \tilde{O}(\sqrt{3T\ln (N/j)}) $ is the $j/N$-regret 
bound for {\ANH}.
So if $j$ is such that $\gamma > \Reg_T^{j/N} / T$, we have
$\frac{1}{T}\sum_{t=1}^T \1\{h_t(\tx_j) = \ty_j \}> \frac{1}{2}$,
which is saying that example $(\tx_j, \ty_j)$ will eventually be classified correctly
by $H(\x)$ due to the fact that $H(\x)$ is taking a majority vote of all $h_t$.
This is in fact true for all examples $(\tx_i, \ty_i)$ such that $i \geq j$
and thus the training error rate will be at most $(j-1)/N$,
which is of order $ \tilde{O}(\exp(-\frac{1}{3}T\gamma^2)) $.

For the margin bound, by plugging $\1\{h_t(\tx_j) = \ty_j \} = (\ty_j h_t(\tx_j)+1)/2$,
we rewrite Eq. \eqref{equ:hedge_and_boosting} as: 
$$ 2\(\gamma - \frac{\Reg_T^{j/N}}{T}\) \leq \frac{1}{T}\sum_{t=1}^T \ty_j h_t(\tx_j). $$
Therefore, if $j$ is such that $\theta < 2(\gamma - \Reg_T^{j/N}/T)$,
then the fraction of examples with margin at most $\theta$ is again at most $(j-1)/N$,
which is of order $\tilde{O}(\exp(-\frac{1}{3}T(\theta-2\gamma)^2))$.
\end{proof}

\section{Experiments in a Boosting Setting}
\label{apd:experiments}
We conducted experiments to compare the performance of
three boosting algorithms for binary classification: AdaBoost \cite{FreundSc97},
NH-Boost (Algorithm \ref{alg:NB}) and {\ANB} (Algorithm \ref{alg:ANB}),
using a set of benchmark data
available from the UCI repository\footnote{\url{http://archive.ics.uci.edu/ml/}} and LIBSVM 
datasets\footnote{\url{http://www.csie.ntu.edu.tw/~cjlin/libsvmtools/datasets/}}.
Some datasets are preprocessed according to \cite{ReyzinSc06}.
The number of features, training examples and test examples can be found
in Table \ref{tab:data}.

All features are binary. The weak learning algorithm is 
a simple (exhaustive) decision stump (see for instance \cite{SchapireFr12}). 
On each round, the weak learning algorithm enumerates all features, 
and for each feature computes the weighted error of the corresponding stump
on the weighted training examples.
Therefore, if the number of examples with zero weight is relatively large,
then the weak learning algorithm would be faster in computing the weighted error
and thus faster in finding the best feature.

All boosting algorithms are run for two hundred rounds.
The results are summarized in Table \ref{tab:results},
with bold entries being the best ones among the three 
(AB, NB and NBDT stand for AdaBoost, NH-Boost and {\ANB} respectively).
As we can see, in terms of training error and test error,
all three algorithms have similar performance.
However, our {\ANB} algorithm is always the fastest one.
The average fraction of examples with zero weights for {\ANB}
is significantly higher than the one for NH-Boost 
(note that AdaBoost does not assign zero weight at all).
We plot the change of this fraction over rounds in Figure \ref{fig:zero}
(using three datasets).
As both algorithms proceed, they tend to ignore more and more examples
on each round, but {\ANB} consistently ignores more examples than NH-Boost.

Since $s_{t,i}$ is positively correlated
to the margin of example $i$ ($\frac{1}{t}\sum_{\tau=1}^t y_i h_\tau(\x_i)$)
and large $s_{t,i}$ leads to zero weight,
the above phenomenon in fact implies that the examples' margins 
should be larger for {\ANB} than for NH-Boost.
This is confirmed by Figure \ref{fig:margins}, where we plot
the final cumulative margins on three datasets
(i.e. each point represents the fraction of examples with at most some fixed margin).
One can see that the lines for {\ANB} are below the ones for 
NH-Boost (and even AdaBoost) for most time,
meaning that {\ANB} achieves larger margins in general.
This could explain {\ANB}'s better test error on some datasets.

\begin{table}[h]
\caption{Description of datasets}
\label{tab:data}
\begin{center}
\begin{tabular}{|c|c|c|c|}
\hline
Data & {\#}feature & {\#}training & {\#}test \\
\hline
a9a & 123 & 32,561 & 16,281 \\
\hline
census & 131 & 1,000 & 1,000 \\
\hline
ocr49 & 403 & 1,000 & 1,000 \\
\hline
splice & 240 & 500 & 500 \\
\hline
w8a & 300 & 49,749	 & 14,951\\
\hline
\end{tabular}
\end{center}
\end{table}

\begin{table}[h]
\caption{Experiment results}
\label{tab:results}
\begin{center}
\begin{tabular}{|c|c|c|c||c|c||c|c|c||c|c|c|}
\hline
& \multicolumn{3}{|c||}{Time (s)} & \multicolumn{2}{|c||}{Zeros (\%)}  &
\multicolumn{3}{|c||}{Training Error (\%)} & \multicolumn{3}{|c|}{Test Error (\%)}  \\
\hline
Data & AB & NB & NBDT & NB & NBDT & AB & NB & NBDT & AB & NB & NBDT \\
\hline
a9a & 57.5 & 72.5 & \textbf{46.2} & 1.1 & \textbf{22.1} &
\textbf{15.4} & 15.8 & 15.5 & \textbf{15.0} & 15.6  & 15.2 \\
\hline
census & 1.7 & 2.2 & \textbf{1.4} & 2.2 & \textbf{19.2}  &
15.6 & 17.0 & \textbf{15.4} & 18.7 & 18.6 & \textbf{18.3} \\
\hline
ocr49 & 5.1 & 4.7 & \textbf{3.0} & 17.1 & \textbf{42.0} &
\textbf{1.7} & \textbf{1.7} & 2.4 & \textbf{5.5} & 5.9 & 5.8 \\
\hline
splice & 1.6 & 1.5 & \textbf{0.9} & 22.2 & \textbf{45.1} & 
\textbf{0.0} & \textbf{0.0} & 0.4 & 9.4 & 8.6 & \textbf{8.2} \\
\hline
w8a & 237.6 & 244.7 & \textbf{170.7} & 3.0 & \textbf{29.3} &
2.6 & \textbf{2.2} & 2.4 & 2.7 & \textbf{2.3} & 2.6 \\
\hline
\end{tabular}
\end{center}
\end{table}

\begin{figure}[h]
\centering
\begin{subfigure}[b]{0.32\textwidth}
\includegraphics[width=\textwidth]{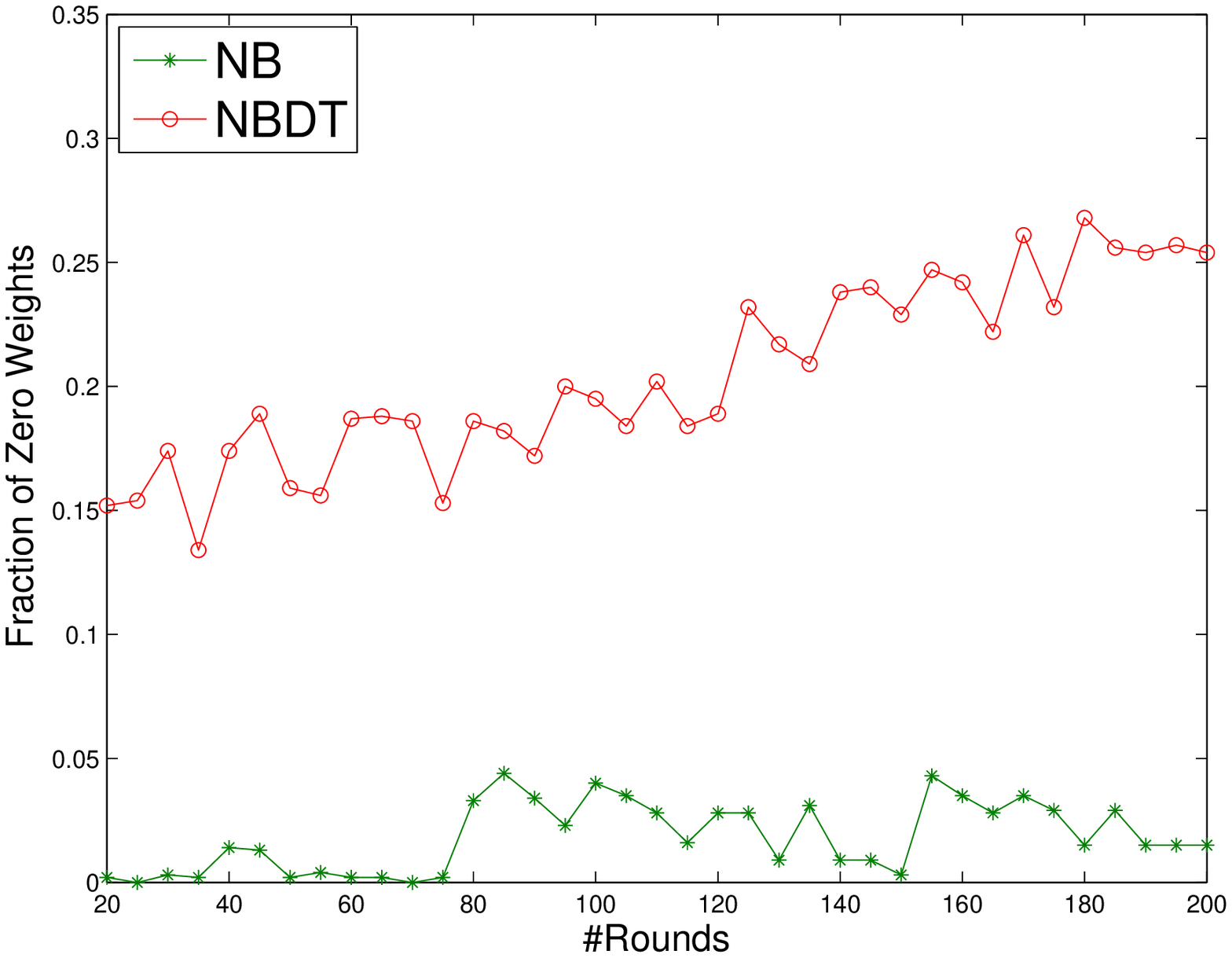}
\caption{census}
\end{subfigure}
\begin{subfigure}[b]{0.32\textwidth}
\includegraphics[width=\textwidth]{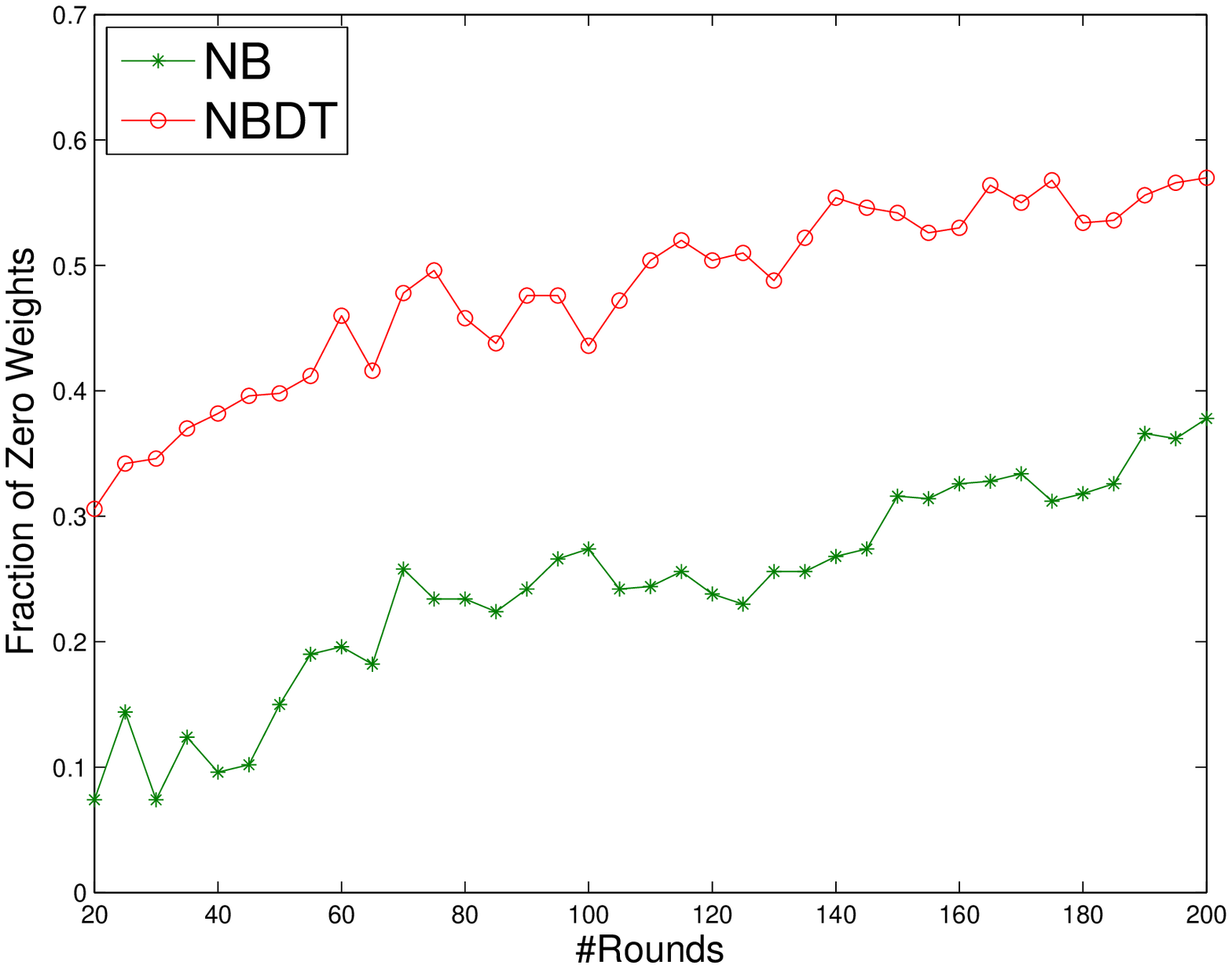}
\caption{splice}
\end{subfigure}
\begin{subfigure}[b]{0.32\textwidth}
\includegraphics[width=\textwidth]{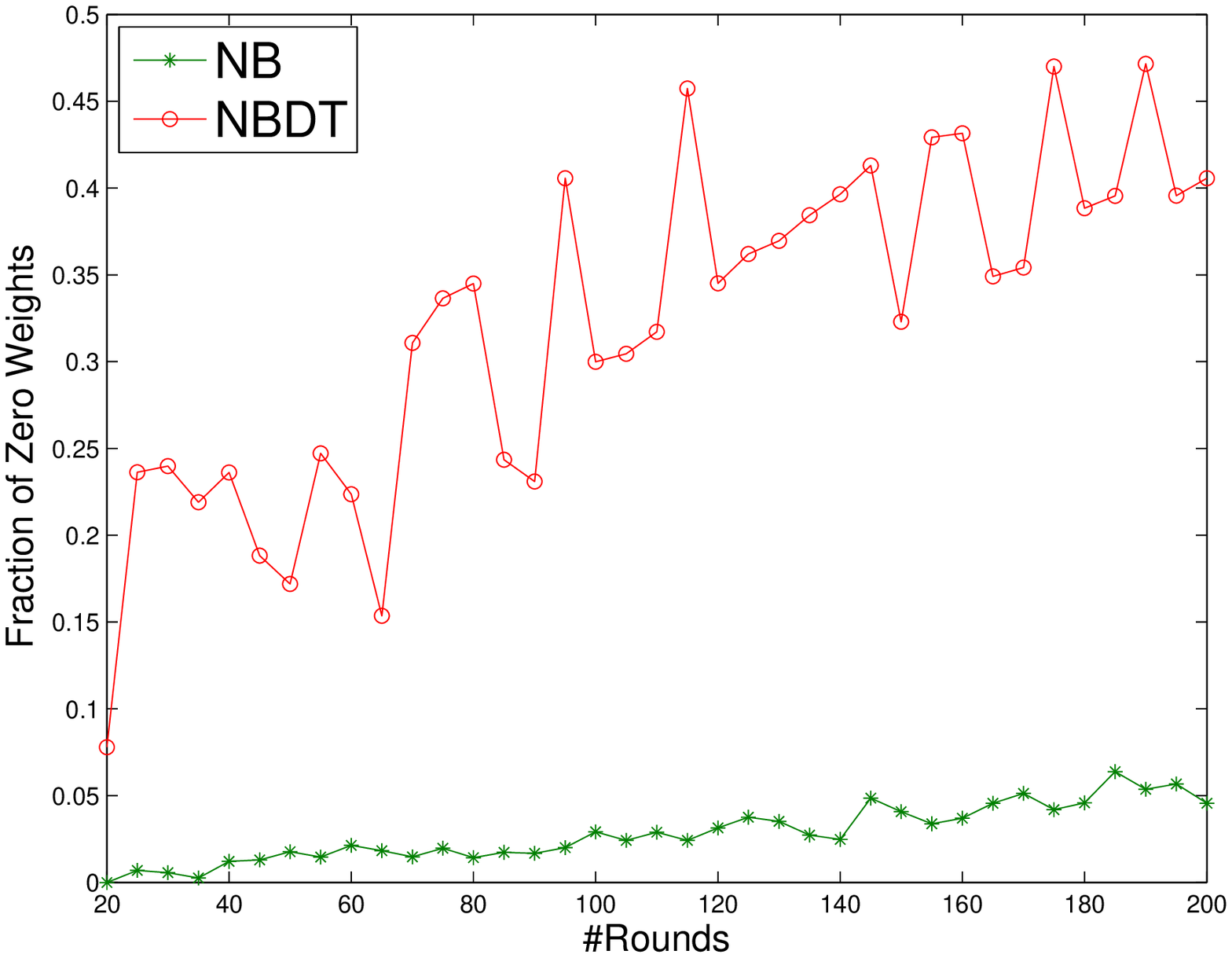}
\caption{w8a}
\end{subfigure}
\caption{Comparison of fraction of zero weights}\label{fig:zero}
\end{figure}

\begin{figure}[h!]
\centering
\begin{subfigure}[b]{0.32\textwidth}
\includegraphics[width=\textwidth]{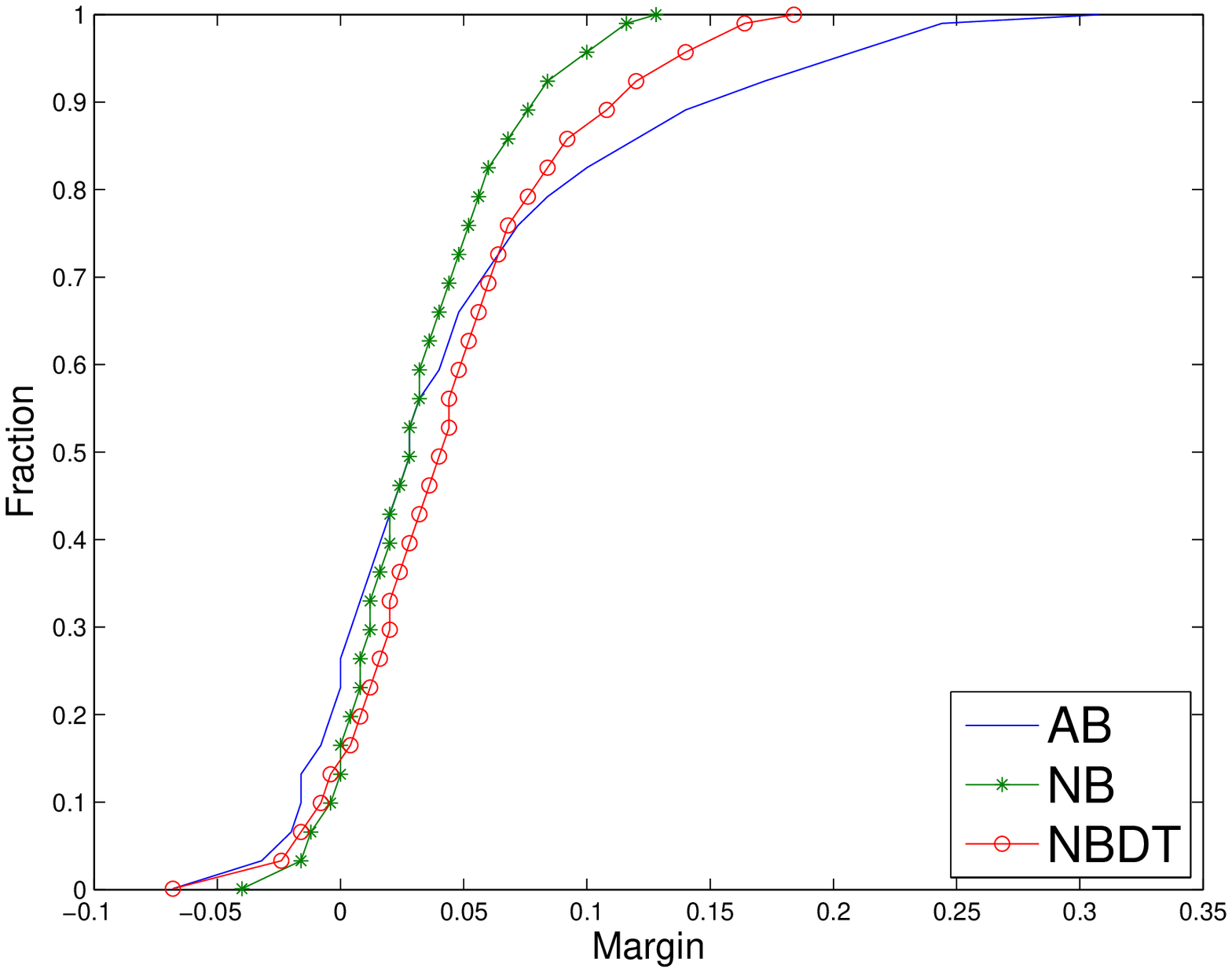}
\caption{census}
\end{subfigure}
\begin{subfigure}[b]{0.32\textwidth}
\includegraphics[width=\textwidth]{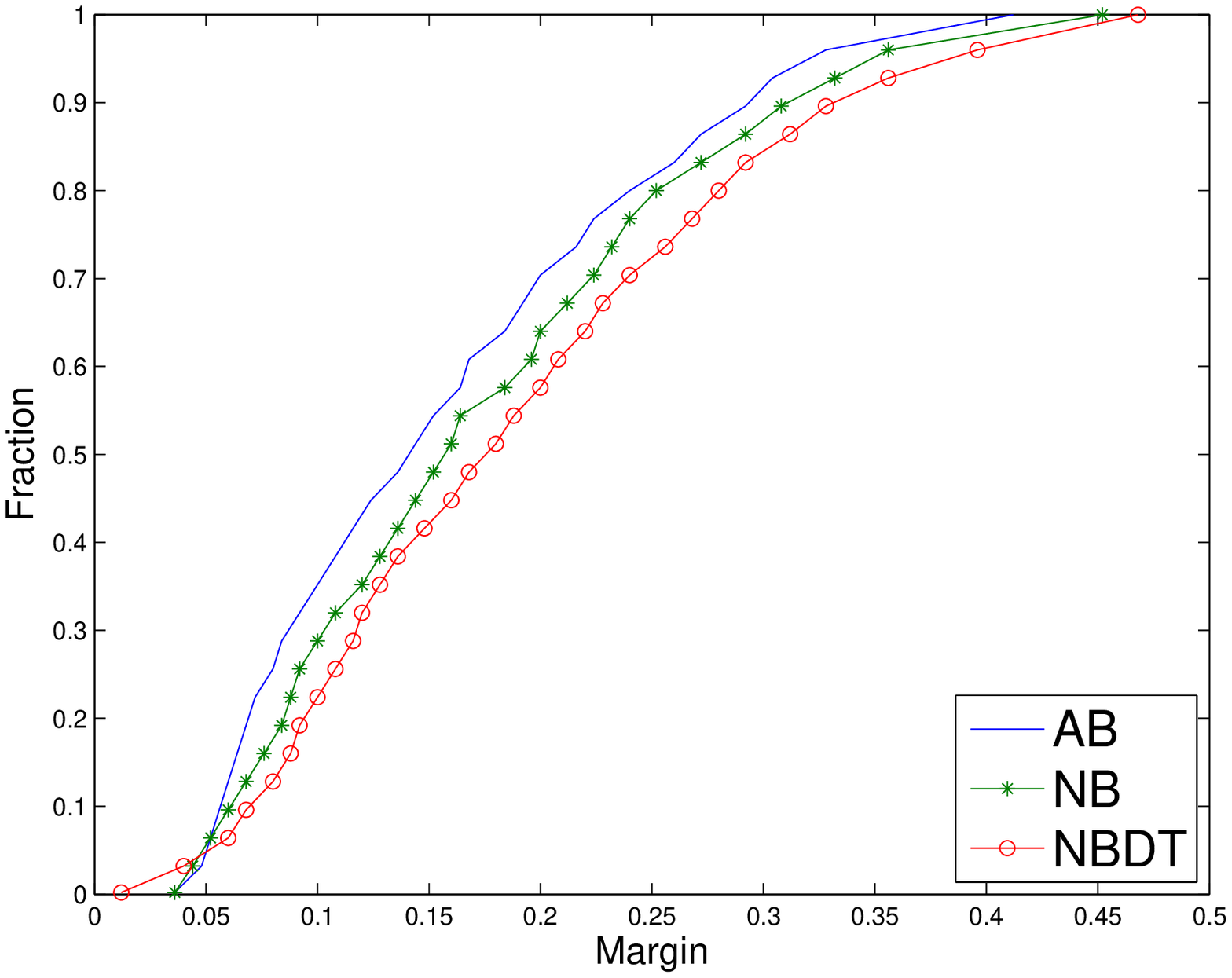}
\caption{splice}
\end{subfigure}
\begin{subfigure}[b]{0.32\textwidth}
\includegraphics[width=\textwidth]{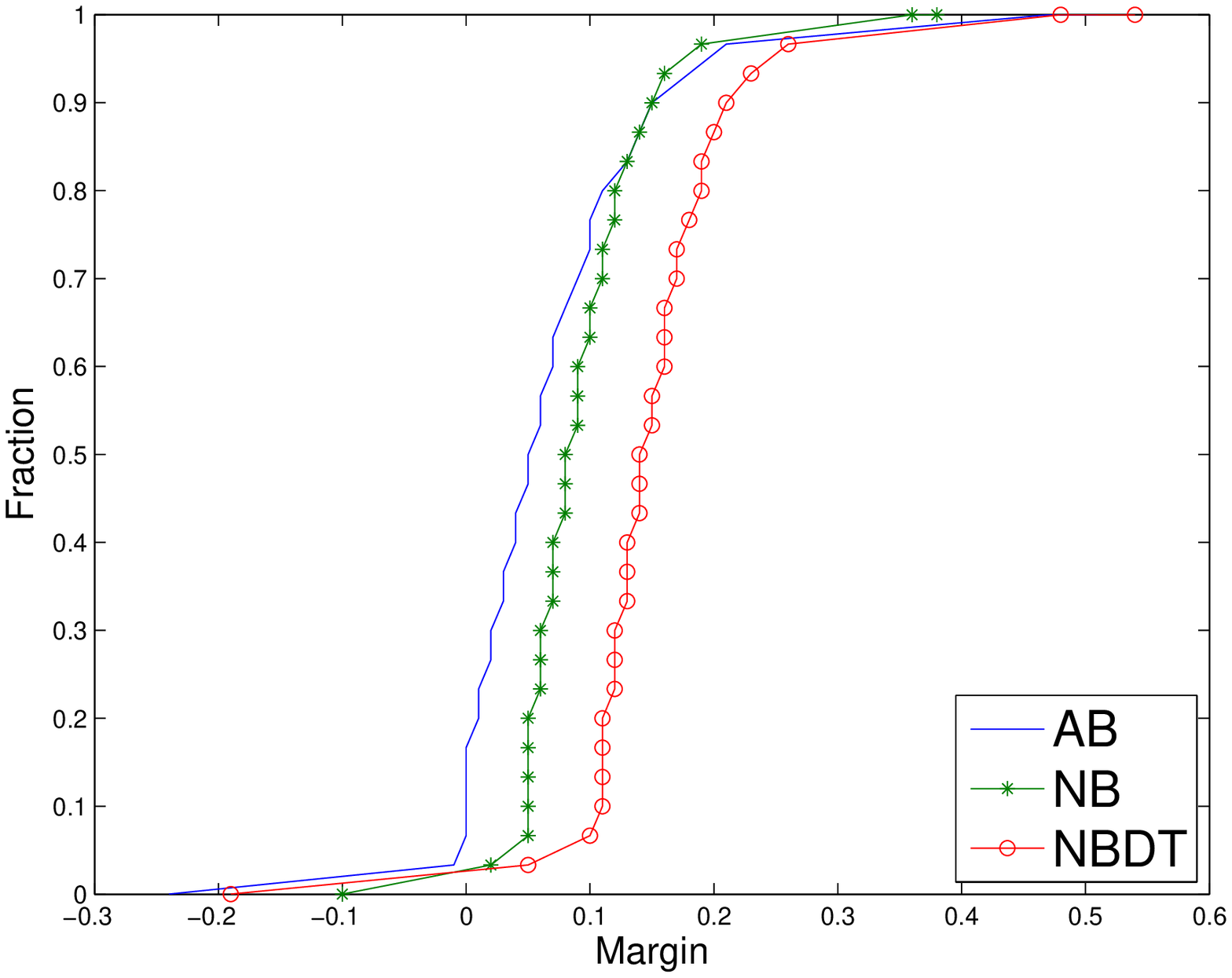}
\caption{w8a}
\end{subfigure}
\caption{Comparison of cumulative margins}\label{fig:margins}
\end{figure}

\end{document}